\documentclass{article}
\usepackage{iclr2027_conference,times}
\usepackage{amsmath,amssymb,amsthm,bm,booktabs,array,graphicx,microtype,placeins,float}
\usepackage{algorithm}
\usepackage{algpseudocode}
\usepackage{hyperref}
\usepackage{url}
\usepackage{caption}
\usepackage{wrapfig}
\hypersetup{hidelinks}
\renewcommand{\eqref}[1]{(\ref{#1})}

\newtheorem{theorem}{Theorem}
\newtheorem{proposition}[theorem]{Proposition}
\newtheorem{corollary}[theorem]{Corollary}

\theoremstyle{remark}

\DeclareMathOperator{\tr}{tr}

\newcommand{\Risk}{\mathcal{R}}
\newcommand{\Value}{\mathcal{V}}

\title{Training-Aware Target Coverage for \\ Synthetic Data Selection}

\author{
Yang Ba$^{1}$, Michelle V. Mancenido$^{2}$ \& Rong Pan$^{1}$\\
$^{1}$School of Computing and Augmented Intelligence, Arizona State University\\
$^{2}$School of Mathematical and Natural Sciences, Arizona State University\\
\texttt{\{yangba,mmanceni,Rong.Pan\}@asu.edu}
}

\iclrfinalcopy

\fancypagestyle{firstpagefooter}{
    \fancyhf{}
    \fancyfoot[L]{Preprint. Under review.}
    \renewcommand{\headrulewidth}{0pt}
    \renewcommand{\footrulewidth}{0pt}
}

\begin{document}
\maketitle
\thispagestyle{firstpagefooter}

\begin{abstract}

Synthetic data are increasingly used to scale LLM training, yet more synthetic data do not necessarily produce better models. Useful synthetic data must add information relevant to the target task without introducing errors that offset their benefit, and the value of an example can change as the training set grows. We develop a linear theory that characterizes this tradeoff and determines where synthetic data are useful, how much should be added, and the marginal value of adding one example to an existing set. The analysis shows the conditions when input coverage alone is sufficient and when synthetic errors must also be considered. Guided by these results, we introduce \emph{Training-Aware Target Coverage} (TATC), a synthetic data selection method for LLM fine-tuning. TATC identifies candidates whose training effects are beneficial to the target task and selects among them to expand coverage of target-relevant directions not already represented by the available data. Experiments on text and image data verify the linear theory. With mathematical reasoning tasks, TATC selects synthetic solutions for fine-tuning Qwen2.5-Math-1.5B-Instruct and outperforms alternative synthetic-data selection methods on GSM8K across selection budgets. In summary, we provide a principled approach to synthetic data selection by quantifying and maximizing its value to the target task. The code is available at https://github.com/YangBa78/Training-Aware-Target-Coverage-for-Synthetic-Data-Selection. 

\end{abstract}

\section{Introduction}
Synthetic data have become an integral part of large language model training, spanning multiple stages of the training pipeline, from pretraining and instruction tuning to reasoning distillation and model evaluation~\citep{gunasekar2023textbooks,wang2023selfinstruct,zelikman2022star}. Models can generate large volumes of synthetic data, such as mathematical solutions, code, instructions, critiques, and preference feedback, with little additional human annotation, making it a scalable way to construct training data. This scalability, however, does not imply that more synthetic data are always better. The value of synthetic examples can differ substantially: adding redundant examples may provide little new information, while examples containing errors can actively harm downstream model performance; these effects can become particularly consequential as synthetic data are incorporated at larger scales or reused across generations of models~\citep{shumailov2024collapse,alemohammad2024mad,yi2026verification}. The central question is therefore no longer simply how to generate more synthetic data, but \textit{which synthetic examples are useful for training, and how many of them should be incorporated into the training dataset?}

Existing approaches to synthetic data selection emphasize complementary properties of training examples. Quality-based methods filter generated responses using correctness or verification signals and model-based quality estimates~\citep{chen2024alpagasus,liu2024deita}, while diversity-based methods favor collections that span distinct regions of representation or model-induced gradient space~\citep{friedman2023vendi,jung2025prismatic}. Other approaches select examples by matching the synthetic and target distributions or by optimizing information-based design objectives~\citep{ash2021bait,rezaei2026highdim}. A separate line of work estimates training utility more directly through gradient similarity or influence on a target objective~\citep{xia2024less,fan2026optimsyn}. These perspectives capture different aspects of useful synthetic data, but they do not by themselves answer the question regarding a synthetic example's contribution to the existing training set and its impact on the target task. A synthetic input may be able to add information in an underrepresented area of data distribution, but its generated response  may introduce a harmful training effect at the same time. Conversely, an example with a favorable training effect may add little information if its local area is already well represented by the real data or previously selected synthetic examples. This motivates us to develop a synthetic data selection method that jointly considers each example's training effect and its marginal information contribution.

We first develop a linear theory to understand when and how much synthetic data are beneficial for learning. Rather than treating each synthetic example as having merely an intrinsic value, the theory characterizes the value of a selected synthetic set through the tradeoff between the information it adds and the training errors it introduces. This perspective reveals three key properties of synthetic data selection. First, value is \emph{directional}: synthetic data are most useful when they add information in directions that matter for the target task but are insufficiently represented by the existing data. Second, value depends on \emph{amount}: along a given direction, the useful amount of synthetic data can be zero, finite, or continue to increase, depending on the balance between additional information and error. Third, value is \emph{marginal}: the benefit of adding a candidate depends on what the real data and previously selected synthetic examples already provide. Together, these results explain why neither data quality nor coverage alone is sufficient to determine which synthetic examples, or how many of them, should be used. We experiment on text and image data to validate these predictions: accounting for synthetic error distinguishes beneficial examples from harmful examples, while the theory also identifies the amount of synthetic data that can be most useful.

Guided by these results, we introduce \textbf{Training-Aware Target Coverage (TATC)}, a practical method for synthetic data selection in LLM fine-tuning. TATC asks two complementary questions for each candidate: whether training on the synthetic example is expected to benefit the target task, and whether its input adds information in target-relevant directions that are not yet well represented. Using a small labeled real probe set, TATC first retains candidates with beneficial training effects and then selects those that provide complementary target coverage. Coverage is evaluated relative to the real data and previously selected synthetic examples, allowing selection to adapt as the training set grows rather than scoring each candidate in isolation. We evaluate TATC by fine-tuning Qwen2.5-Math-1.5B-Instruct~\citep{yang2024qwenmath} with synthetic mathematical solutions and evaluating on GSM8K~\citep{cobbe2021gsm8k}. Our approach improves exact-match accuracy over both the real data checkpoint and other competing selection methods across a range of selection budgets.

Our contributions can be summarized below:
\begin{itemize}

\item We develop a theoretical framework for characterizing the value of synthetic data, showing how downstream benefit depends jointly on the information added by synthetic inputs and the errors they introduce. The analysis shows \emph{where} synthetic data are useful across feature directions, \emph{how much} data should be added, and \emph{why} the value of each additional example depends on the existing training dataset.

\item Guided by this theory, we introduce \textbf{Training-Aware Target Coverage (TATC)}, a synthetic data selection method in LLM fine-tuning. TATC identifies examples with beneficial training effects and selects candidates that provide complementary coverage of target-relevant directions, updating their value as the selected set grows.

\item We empirically validate the proposed theory on text and image data and show that TATC outperforms existing synthetic data selection methods on GSM8K mathematical reasoning.

\end{itemize}

\label{sec:introduction}

\section{Related Work}
\paragraph{Synthetic inputs and generated response error.}
Recent pipelines construct synthetic training data through rationale distillation, instruction evolution, mathematical problem synthesis, error-driven generation, and diversified instruction generation~\citep{hsieh2023distilling,xu2024wizardlm,yu2024metamath,wang2023synthesize,li2024glan,xu2025magpie,ge2024personahub}; synthetic augmentation has also been studied for calibration and generalization~\citep{ba2024fillgaps}. Recursive synthetic training can amplify model errors~\citep{shumailov2024collapse,alemohammad2024mad}, while verification can mitigate degradation when generated responses can be checked~\citep{yi2026verification}. High-dimensional regression analyses identify covariance shift as a factor in synthetic-data utility and motivate covariance matching~\citep{rezaei2026highdim}. We study selection from a candidate pool of LLM-generated input-response pairs, jointly accounting for input information, systematic response error, and excess response noise.

\paragraph{Quality and validation-guided data selection.}
AlpaGasus filters instruction data using model-based quality scores~\citep{chen2024alpagasus}, while DEITA combines complexity, quality, and diversity~\citep{liu2024deita}. GLISTER selects subsets for a held-out validation objective~\citep{killamsetty2021glister}, and RHO-LOSS prioritizes examples by reducible loss~\citep{mindermann2022rholoss}. Our setting instead separates the predicted training effect of a complete example from the set-dependent target coverage of its features.

\paragraph{Diversity and experimental design.}
The Vendi Score measures set diversity through the spectrum of a similarity matrix~\citep{friedman2023vendi}; G-Vendi applies the same idea to model-induced gradients for synthetic reasoning data generation~\citep{jung2025prismatic}. These methods reward dispersion within a set. In contrast, experimental design targets reductions in estimation error; BAIT makes a Fisher inverse-trace criterion practical for neural active learning~\citep{ash2021bait}. Our target coverage term follows this design principle and is evaluated as an ablation rather than a new objective. Unlike active learning, however, synthetic data selection observes a generated response rather than acquiring a trusted label after selecting an input. Because that response can be systematically wrong, its value must also depend on response error.

\paragraph{Gradient and influence selection.}
Influence functions and Data Shapley scores quantify the contribution of individual training examples~\citep{koh2017influence,ghorbani2019datashapley}. LESS uses optimizer-aware gradient similarity for targeted instruction tuning~\citep{xia2024less}, while OptimSyn uses estimated training influence to guide synthetic-data generation~\citep{fan2026optimsyn}. SPICE combines a Fisher-information objective with gradient conflict to select informative and compatible instruction data~\citep{chang2026spice}. Our analysis instead asks how input information and generated-response error jointly determine the marginal value of adding a synthetic example to the existing training set.

\label{sec:related-work}

\section{Downstream Risk Theory in Synthetic Data Selection}
\label{sec:methodology}

\subsection{Risk of Training with Synthetic Data}

We study when synthetic data improve downstream performance by accounting for both the information provided by synthetic inputs and errors in their labels. Let \(\bm\phi(X)\in\mathbb R^d\) denote a feature representation of input \(X\), let \(Y\) denote the corresponding label, and consider a linear downstream model
$
f_{\bm\theta}(X)=\bm\phi(X)^\top\bm\theta, \bm\theta\in\mathbb R^d.
$
We assume that there exists a true parameter \(\bm\theta^\star\in\mathbb R^d\) such that
$$
\mathbb E[Y\mid X]=\bm\phi(X)^\top\bm\theta^\star.
$$
For the real and synthetic samples, let
$
\{(\bm x_{r,i}, y_{r,i})\}_{i=1}^{n},
\{(\bm x_{s,j}, y_{s,j})\}_{j=1}^{m},
$
denote their feature-label pairs, where
\(\bm x_{r,i},\bm x_{s,j}\in\mathbb R^d\) are obtained from the representation \(\bm\phi\) defined above (i.e., $\bm x=\bm\phi(X)$).

Let
$$
\bm X_r
=
[\bm x_{r,1},\ldots,\bm x_{r,n}]^\top
\in\mathbb R^{n\times d},
\qquad
\bm X_s
=
[\bm x_{s,1},\ldots,\bm x_{s,m}]^\top
\in\mathbb R^{m\times d}
$$
be feature matrices. We model their labels as
\[
\bm y_r=\bm X_r\bm\theta^\star+\bm\varepsilon_r,
\qquad
\bm y_s=\bm X_s\bm\theta^\star+\bm d_s+\bm\varepsilon_s,
\]
where \(\bm d_s\in\mathbb R^m\) represents systematic error in the synthetic labels. Conditional on the feature matrices, \(\bm\varepsilon_r\) and \(\bm\varepsilon_s\) are independent noises with zero means and covariances \(\bm\Omega_r\) and \(\bm\Omega_s\), respectively.

We estimate \(\bm\theta\) from the combined real and synthetic data by least squares:
\[
\widehat{\bm\theta}_s
=
\arg\min_{\bm\theta\in\mathbb R^d}
\left\{
\frac12\|\bm X_r\bm\theta-\bm y_r\|_2^2
+
\frac12\|\bm X_s\bm\theta-\bm y_s\|_2^2
\right\}.
\]
We evaluate the fitted model on a downstream distribution \(P_{\mathrm{test}}\) and define its second moment matrix as
$
\bm Q
:=
\mathbb E_{X\sim P_{\mathrm{test}}}
\left[
\bm\phi(X)\bm\phi(X)^\top
\right]
\succeq 0.
$
We define the expected squared error relative to the conditional mean on \(P_{\mathrm{test}}\) as the \emph{downstream risk}:
$$
\Risk_Q(\widehat{\bm\theta}_s)
:=
\mathbb E_{X\sim P_{\mathrm{test}}}
\left[
\left(
\bm\phi(X)^\top
(\widehat{\bm\theta}_s-\bm\theta^\star)
\right)^2
\right]
=
(\widehat{\bm\theta}_s-\bm\theta^\star)^\top
\bm Q
(\widehat{\bm\theta}_s-\bm\theta^\star).
$$
To characterize the downstream risk, define
\[
\bm A
=
\bm X_r^\top\bm X_r+\bm X_s^\top\bm X_s,
\qquad
\bm u
=
\bm X_s^\top\bm d_s,
\qquad
\bm N
=
\bm X_r^\top\bm\Omega_r\bm X_r
+
\bm X_s^\top\bm\Omega_s\bm X_s.
\]
Here, \(\bm A\) is the feature Gram matrix, \(\bm u\) represents the contribution of systematic label error, and \(\bm N\) represents the contribution of random label noise.

\begin{proposition}[Downstream risk decomposition]
\label{prop:master-risk}
Assume that \(\bm A\) is nonsingular. Conditional on \((\bm X_r,\bm X_s)\), under the noise assumptions above,
\begin{equation}
\boxed{
\mathbb E\!\left[
\Risk_Q(\widehat{\bm\theta}_s)
\mid
\bm X_r,\bm X_s
\right]
=
\underbrace{
(\bm A^{-1}\bm u)^\top\bm Q(\bm A^{-1}\bm u)
}_{\text{bias from systematic label error}}
+
\underbrace{
\operatorname{tr}\!\left(
\bm Q\bm A^{-1}\bm N\bm A^{-1}
\right)
}_{\text{variance from label noise}}
}
\label{eq:master-risk}
\end{equation}
\end{proposition}
Proposition~\ref{prop:master-risk} gives a fundamental decomposition used throughout this analysis. Synthetic inputs affect downstream risk through the information matrix \(\bm A\), while errors in synthetic labels affect the risk through both systematic error and random variation. Thus, the benefit of adding synthetic data cannot be decided by input coverage alone. See proof in Appendix~\ref{app:master-proof}.

\begin{figure*}[t]
\centering
\includegraphics[width=.98\textwidth]{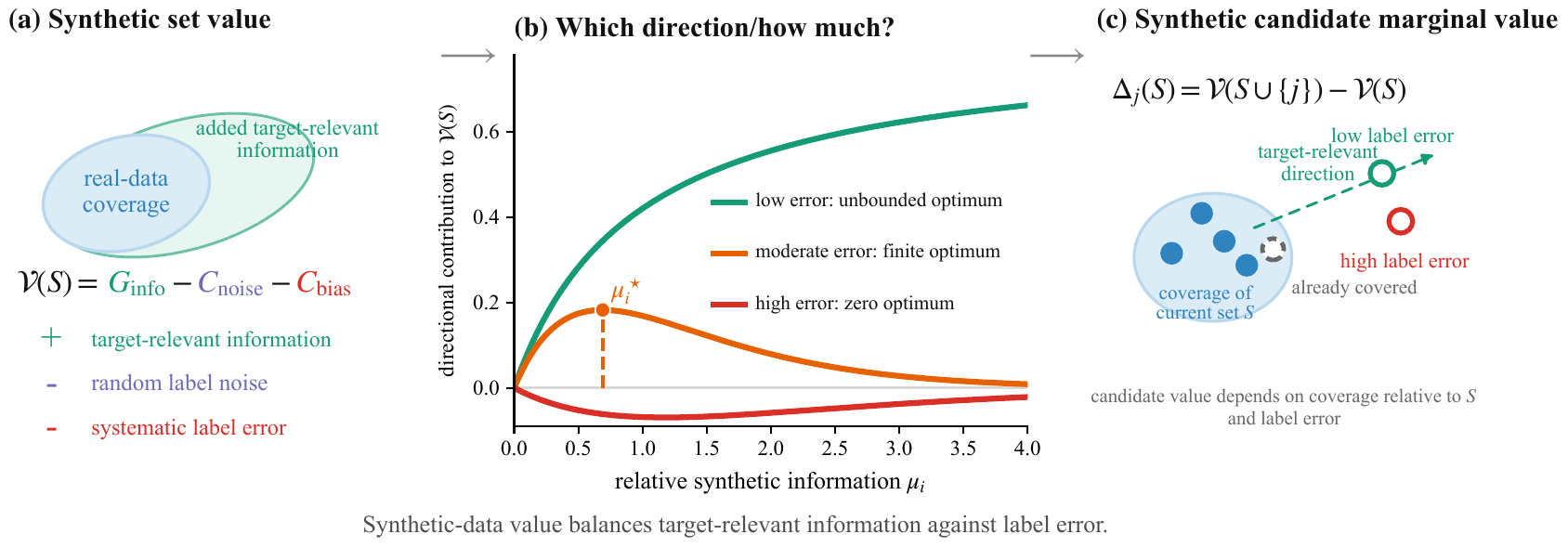}
\caption{\textbf{Risk and value of synthetic data.} The theory in one view. (a) A synthetic set is valuable when target-relevant input information outweighs random label noise and systematic label error. (b) This tradeoff determines whether the optimal amount along a direction is zero, finite, or unbounded. (c) Candidate value depends on both input coverage relative to the current set and label error.}
\label{fig:theory-overview}
\end{figure*}

\subsection{Value of Synthetic Data}

We now move from analyzing a given synthetic dataset to selecting examples from a synthetic candidate pool. Let
$
\mathcal P_s=\{(\bm x_j,y_j)\}_{j=1}^N
$
denote the candidate pool, and let \(S\subseteq\{1,\ldots,N\}\) index a selected synthetic subset. We treat \(S\) as fixed with respect to the realized label noise and use \(\bm X_S,\bm y_S,\bm d_S,\) and \(\bm\Omega_S\) for the corresponding quantities restricted to \(S\). We define the value of \(S\) as the reduction in expected downstream risk relative to training on the real data alone:
$$
\Value(S)
:=
\mathbb E[
\Risk_Q(\widehat{\bm\theta}_r)
\mid\bm X_r]
-
\mathbb E[
\Risk_Q(\widehat{\bm\theta}_S)
\mid\bm X_r,\bm X_S].
$$
Thus, \(\Value(S)>0\) means that adding \(S\) improves expected downstream performance. To separate the sources of this added value, we consider isotropic label noise,
$$
\bm\Omega_r=\sigma^2\bm I,
\qquad
\bm\Omega_S=(\sigma^2+\tau^2)\bm I,
$$
where \(\tau^2\ge0\) is the excess random noise in the synthetic labels. Let
$$
\bm A_r=\bm X_r^\top\bm X_r\succ0,
\qquad
\bm A_S=\bm X_S^\top\bm X_S,
\qquad
\bm A=\bm A_r+\bm A_S
$$

\begin{corollary}[Value of a selected synthetic set]
\label{cor:set-value}
Under the specifications above,
\begin{equation}
\boxed{
\begin{aligned}
\Value(S)={}&
\underbrace{\sigma^2\tr[\bm Q(\bm A_r^{-1}-\bm A^{-1})]}_{G_{\rm info}(S)}
-\underbrace{\tau^2\tr(\bm Q\bm A^{-1}\bm A_S\bm A^{-1})}_{C_{\rm noise}(S)}\\
&-\underbrace{(\bm A^{-1}\bm X_S^\top\bm d_S)^\top
\bm Q(\bm A^{-1}\bm X_S^\top\bm d_S)}_{C_{\rm bias}(S)} .
\end{aligned}}
\label{eq:set-value}
\end{equation}
\end{corollary}

Equation~\eqref{eq:set-value} expresses the value of a selected synthetic set as a tradeoff between the information provided by its inputs and errors in its labels. \(G_{\mathrm{info}}(S)\) measures the information relevant to the downstream task that is added beyond the real data, \(C_{\mathrm{noise}}(S)\) is the downstream cost of excess random label noise, and \(C_{\mathrm{bias}}(S)\) is the downstream cost induced by systematic error in the synthetic labels. Thus, adding \(S\) improves downstream performance when its information gain exceeds the two error terms. Figure~\ref{fig:theory-overview} summarizes this tradeoff and the directional and candidate-level results developed below. See Appendix~\ref{app:set-value-proof} for the proof.

\paragraph{Which directions are useful?}
To understand where the value in Equation~\eqref{eq:set-value} comes from, we compare the information provided by the selected synthetic inputs with what was already provided by the real data. By spectral analysis, define
\begin{equation*}
\bm A_r^{-1/2}\bm A_S\bm A_r^{-1/2}
=\bm R\operatorname{diag}(\mu_1,\ldots,\mu_d)\bm R^\top,
\quad
\bm D=\bm A_r^{-1/2}\bm Q\bm A_r^{-1/2},
\end{equation*}
where \(\bm R=[\bm r_1,\ldots,\bm r_d]\) is orthogonal. Let
$
w_i=\bm r_i^\top\bm D\bm r_i.
$
The eigenvalue \(\mu_i\ge0\) measures the synthetic information added along direction \(\bm r_i\) relative to the real data, while \(w_i\ge0\) measures the downstream relevance of that direction.

\begin{proposition}[Directional value of synthetic information]
\label{prop:modes}
Assume \(\bm A_r\succ0\). The information gain and excess noise terms in Corollary~\ref{cor:set-value} decompose as
$$
G_{\mathrm{info}}(S)
=
\sigma^2\sum_{i=1}^d
w_i\frac{\mu_i}{1+\mu_i},
\qquad
C_{\mathrm{noise}}(S)
=
\tau^2\sum_{i=1}^d
w_i\frac{\mu_i}{(1+\mu_i)^2}.
$$
These decompositions hold without further assumptions. If, in addition, the systematic label error has the form
$
\bm d_S=\bm X_S\bm\delta,
\bm\delta\in\mathbb R^d,
$
and \(\bm R^\top\bm D\bm R\) is diagonal, define
$
\bm\eta
=
\bm R^\top\bm A_r^{1/2}\bm\delta.
$
Then the full set value decomposes across directions:
\begin{equation}
\Value(S)
=
\sum_{i=1}^d
w_i
\left[
\sigma^2\frac{\mu_i}{1+\mu_i}
-
\tau^2\frac{\mu_i}{(1+\mu_i)^2}
-
\eta_i^2
\left(\frac{\mu_i}{1+\mu_i}\right)^2
\right].
\label{eq:mode-value}
\end{equation}
\end{proposition}

Equation~\eqref{eq:mode-value} simplifies the tradeoff for the set in Equation~\eqref{eq:set-value} direction by direction. Each direction contributes according to the information added by the synthetic input, its downstream relevance, and the random and systematic errors in the synthetic label.

\begin{corollary}[Beneficial directions]
\label{cor:beneficial-direction}
Under the conditions of Proposition~\ref{prop:modes}, a direction with
\(\mu_i>0\) and \(w_i>0\) contributes positively to \(\Value(S)\) if and only if
$$
\eta_i^2
<
\sigma^2
+
\frac{\sigma^2-\tau^2}{\mu_i}.
$$
\end{corollary}
Thus, a direction relevant to the downstream task is not necessarily beneficial. Its information gain must be sufficient to offset the random and systematic errors associated with the synthetic labels. See Appendix~\ref{app:beneficial-proof} for the proof.

\begin{corollary}[Optimal amount of synthetic information]
\label{cor:optimal-amount}
Under the conditions of Proposition~\ref{prop:modes}, let \(w_i>0\). If
$$
\sigma^2>\tau^2,\qquad
2\eta_i^2>\sigma^2+\tau^2,
$$
the directional contribution is maximized at
\begin{equation}
\mu_i^\star
=
\frac{\sigma^2-\tau^2}
{2\eta_i^2-\sigma^2-\tau^2}.
\label{eq:amount-optimum}
\end{equation}
Otherwise, the optimum lies at a boundary: \(\mu_i^\star=0\) when
\(\sigma^2<\tau^2\) and \(\eta_i^2\ge\sigma^2\), and the supremum is approached as \(\mu_i\to\infty\) in the complementary nondegenerate regime. Equality and degenerate cases are given in Appendix~\ref{app:amount-proof}.
\end{corollary}
Corollary~\ref{cor:optimal-amount} shows that more synthetic information is not always beneficial: the optimal amount along a direction can be zero, finite, or unbounded. Equation~\eqref{eq:amount-optimum} varies \(\mu_i\) while keeping \(w_i\) and \(\eta_i\) unchanged, and therefore characterizes the optimal amount along a given direction rather than which synthetic examples should be selected.

The directional analysis above assumes that the synthetic label error can be assessed. This requirement cannot generally be met in directions that are not covered by the real data. It also assumes \(\bm A_r\succ0\), so the real data identify every modeled direction. We next consider the complementary case in which a downstream-relevant direction is not covered by the real training data.

\begin{proposition}[When synthetic label error is not identifiable]
\label{prop:nonidentifiability}
Let \(\bm v\in\mathbb R^d\) satisfy
$
\bm X_r\bm v=\bm0,
\bm X_S\bm v\neq\bm0,
\bm v^\top\bm Q\bm v>0.
$
Without additional labeled information identifying
\(\bm v^\top\bm\theta^\star\), the observations cannot distinguish
\((\bm\theta^\star,\bm d_S)\) from
\((\bm\theta^\star+c\bm v,\bm d_S-c\bm X_S\bm v)\).
\end{proposition}
Thus, when a downstream-relevant direction is absent from the real training data, synthetic labels alone cannot distinguish target signal from systematic label error. Additional labeled real data covering such directions can resolve this ambiguity; Section~\ref{sec:dl} uses a real probe set for this purpose.

\subsection{Marginal Value of a Synthetic Example}
\label{sec:marginal}

The value in Equation~\eqref{eq:set-value} determines whether a selected synthetic set is beneficial as a whole. We now return to the general risk formulation and ask how this value changes when one candidate is added to the current set \(S\). For a synthetic candidate \(j\notin S\), define
$$
\Delta_j(S)
:=
\Value(S\cup\{j\})-\Value(S).
$$
Thus, \(\Delta_j(S)>0\) means that adding candidate \(j\) further reduces expected downstream risk.
Let \((\bm A,\bm u,\bm N)\) denote the information, error, and noise quantities defined before Proposition~\ref{prop:master-risk}, now formed from the real data and selected set \(S\). Let \(\bm x_j\), \(\delta_j\), and \(\omega_j\) denote the candidate's feature vector, systematic label error, and random label noise variance, respectively, with its noise independent of the real and selected-sample noises. Define
$$
\bm b=\bm A^{-1}\bm u,\qquad
\bm p_j=\bm A^{-1}\bm x_j,\qquad
c_j=1+\bm x_j^\top\bm p_j,\qquad
e_j=\delta_j-\bm x_j^\top\bm b,\qquad
q_j=\bm p_j^\top\bm Q\bm p_j.
$$
Here, \(\bm b\) is the bias induced by the selected set, \(e_j\) is the candidate's systematic error relative to that bias, and \(q_j\) measures the downstream relevance of its induced direction.

\begin{theorem}[Conditional marginal value]
\label{thm:rank-one-marginal}
The marginal value of candidate \(j\) is
\begin{equation}
\boxed{\begin{aligned}
\Delta_j(S)={}&
\frac{2\bm p_j^\top\bm Q\bm A^{-1}\bm N\bm p_j}{c_j}
-\frac{q_j(\bm p_j^\top\bm N\bm p_j+\omega_j)}{c_j^2}\\[-2pt]
&-\frac{2e_j\bm b^\top\bm Q\bm p_j}{c_j}
-\frac{e_j^2q_j}{c_j^2}.
\end{aligned}}
\label{eq:rank-one-marginal}
\end{equation}
\end{theorem}
The first two terms describe the change in the variance component of downstream risk, while the last two describe the impact of systematic label error and its interaction with the bias already induced by \(S\). Therefore, the value of a candidate is conditional on the selected set: the same candidate can have different marginal value depending on the examples already included.

A useful special case isolates the contribution of input coverage. If
$
\bm\Omega_r=\bm\Omega_S=\sigma^2\bm I,
\bm b=\bm0,
\delta_j=0,
\omega_j=\sigma^2,
$
then
\begin{equation}
\Delta_j(S)=\frac{\sigma^2q_j}{c_j}.
\label{eq:coverage-marginal}
\end{equation}
The more general specializations with systematic label error or excess random noise are given in Appendix~\ref{app:rank-one-proof}. Equation~\eqref{eq:coverage-marginal} isolates the coverage relevant to the downstream task that is contributed by the candidate input. Theorem~\ref{thm:rank-one-marginal} shows that this coverage term is sufficient only when labels are clean; more generally, candidate value also depends on label error and its interaction with the current bias. This insight motivates the development of TATC in Section~\ref{sec:dl}.

\begin{figure}[t]
\centering
\includegraphics[width=.99\textwidth]{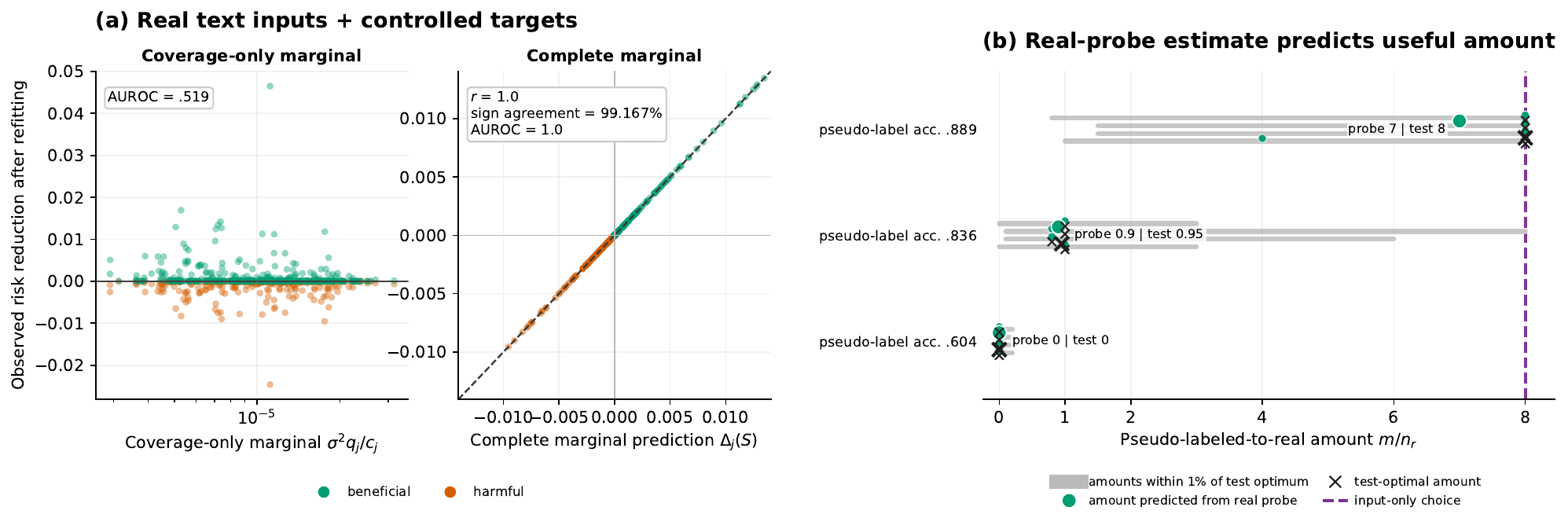}
\caption{\textbf{Candidate value and synthetic-data amount with real inputs and controlled targets.}
\textbf{(a)} Candidate-level test on SetFit/subj. Coverage alone does not distinguish beneficial from harmful candidates, whereas the complete marginal closely predicts the risk reduction from explicit refitting.
\textbf{(b)} Amount selection with ResNet-50 features on CIFAR-10. The amount predicted from a labeled real probe closely tracks the test-optimal amount and decreases as pseudo-label reliability falls, while the input-only rule selects the full candidate pool. Gray segments indicate amounts within 1\% of the test minimum.}
\label{fig:exact-regime}
\end{figure}

Figure~\ref{fig:exact-regime} tests two implications of the linear theory using semi-synthetic constructions that keep candidate inputs real while changing their labels. This lets us examine label error without changing input information, which would change simultaneously if both inputs and labels were generated. In panel (a), controlled label errors are assigned to SetFit/subj texts~\citep{pang2004sentimental}; coverage alone fails to distinguish beneficial from harmful candidates, whereas the complete marginal closely predicts the risk reduction observed after refitting with each candidate. In panel (b), the same CIFAR-10~\citep{krizhevsky2009cifar} candidate images receive labels from models with different error levels; the probe-estimated useful amount closely tracks the test optimum as label error increases, while the input-only criterion selects the full candidate pool. Together, these tests show that changing labels alone can change both which examples are useful and how much data should be used. Section~\ref{sec:dl} extends this distinction to nonlinear fine-tuning; details are in Appendix~\ref{app:exact-tests}.

\section{Training-Aware Target Coverage in Supervised Fine-tuning}
\label{sec:dl}

In this section, we translate the linear analysis in Section~\ref{sec:methodology} into a practical synthetic-data selection method for nonlinear models. The theory shows that input coverage alone is insufficient because synthetic responses can introduce error. Using LLM fine-tuning as a concrete setting, we combine target coverage with a local estimate of each complete synthetic example's training effect.

We start from a model checkpoint \(\bm \theta_0\) obtained by fine-tuning on real data. Let
$
\mathcal D_{\mathrm{syn}}
=
\{(x_j,y_j)\}_{j=1}^{N}
$
denote the synthetic candidate pool.
Let
$
\bm\phi_{\bm \theta_0}(x)\in\mathbb R^p
$
denote the representation of input \(x\) extracted at \(\bm \theta_0\), and write
\(\bm\phi_j=\bm\phi_{\bm \theta_0}(x_j)\).
We use the real training set \(\mathcal D_r\) together with a disjoint real probe set \(\mathcal D_q\). The real training set determines the checkpoint and initial representation geometry, while the probe set provides target-task inputs and labels for measuring downstream relevance and the predicted effect of synthetic examples.

\paragraph{Target coverage.}
We first measure which representation directions are covered by the real training data and which directions occur in the labeled probe set. Using the representations defined above, let
\[
\begin{gathered}
\widehat{\bm A}_0 = \rho\bm I + \sum_{x\in\mathcal D_r}
\bm\phi_{\bm \theta_0}(x)\bm\phi_{\bm \theta_0}(x)^\top,
\qquad
\widehat{\bm Q} = \frac{1}{|\mathcal D_q|}
\sum_{x\in\mathcal D_q}
\bm\phi_{\bm \theta_0}(x)\bm\phi_{\bm \theta_0}(x)^\top.
\end{gathered}
\]
Here, \(\rho>0\) ensures numerical stability of the empirical Gram matrix.
For a selected set of candidate indices \(S\subseteq\{1,\ldots,N\}\), define
\[
\begin{gathered}
\widehat{\bm A}_S
=
\widehat{\bm A}_0
+
\sum_{j\in S}\bm\phi_j\bm\phi_j^\top.
\end{gathered}
\]
For a candidate \(j\notin S\), its marginal coverage gain is
\[
\begin{gathered}
a_j(S)
=
\frac{
\bm\phi_j^\top \widehat{\bm A}_S^{-1}
\widehat{\bm Q}
\widehat{\bm A}_S^{-1}\bm\phi_j
}{
1+\bm\phi_j^\top\widehat{\bm A}_S^{-1}\bm\phi_j
}.
\end{gathered}
\]
By Sherman--Morrison, \(a_j(S)\) is exactly the reduction in \(\tr(\widehat{\bm Q}\widehat{\bm A}_S^{-1})\) from adding \(x_j\), which we call its \emph{target coverage}.
It measures how much candidate input \(x_j\) adds coverage in directions represented by the real probe set but not already covered by the existing training data. Because \(a_j(S)\) depends only on \(x_j\), it does not account for the effect of its synthetic response \(y_j\).

\paragraph{Predicted local training effect.}
Input coverage alone does not determine whether a synthetic example is useful: two candidates with similarly informative inputs can induce different parameter updates because their synthetic responses differ. We estimate this effect by predicting how the candidate's optimizer update changes loss on the real probe.

For a synthetic sample ($x_j,y_j$), let
$
\Delta\bm\theta_j
$
denote the difference between the AdamW parameter displacements obtained using the candidate gradient and a zero gradient, both computed from the same checkpoint \(\bm\theta_0\) and saved optimizer state \citep{loshchilov2019adamw}. The construction is detailed in Appendix~\ref{app:optimizer}. Let \(\bm g_q\) and \(\bm H_q\) denote the gradient and a local curvature approximation of the loss on the probe set \(\mathcal D_q\), evaluated at the same checkpoint. In our implementation, \(H_q\) is a diagonal empirical-Fisher approximation \citep{kunstner2019empiricalfisher}.
A second-order Taylor approximation gives
\begin{equation}
h_j
=
-\bm g_q^\top\Delta\bm\theta_j
-\frac{1}{2}
\Delta\bm\theta_j^\top
\bm H_q
\Delta\bm\theta_j .
\label{eq:local-value}
\end{equation}
A larger \(h_j\) indicates a larger predicted reduction in loss on the real probe. Unlike \(a_j(S)\), \(h_j\) depends on the complete synthetic example through its induced parameter update. It therefore provides a fixed-checkpoint score for the response-dependent effect induced by the candidate's local optimizer update; this effect is absent from the coverage criterion.

\paragraph{Training Aware Target Coverage (TATC).}

The two quantities play complementary roles in Algorithm~\ref{alg:tatc}: \(h_j\) ranks candidates by their predicted local training effect, while \(a_j(S)\) selects among them based on the additional target coverage of \(x_j\) given the examples already selected.

\begin{wrapfigure}{r}{0.55\columnwidth}
\vspace{-1.2\baselineskip}
\centering
\small

\hrule
\vspace{5pt}
{\captionsetup{aboveskip=0pt,belowskip=0pt}
\captionof{algorithm}{Training-Aware Target Coverage}
\label{alg:tatc}}
\vspace{2pt}
\hrule
\vspace{5pt}

\begin{algorithmic}[1]
\Require Synthetic candidates \(\{(x_j,y_j)\}_{j=1}^N\),
predicted local training effects \(\{h_j\}_{j=1}^N\),
budget \(B\), expansion factor \(\kappa\)
\State \(S_0\gets\varnothing\); rank candidates by decreasing \(h_j\)
\For{\(s=1,\ldots,B\)}
    \State \(\mathcal E_s\gets\) top \(\min\{\kappa s,N\}\) candidates by \(h_j\),
    excluding \(S_{s-1}\)
    \State \(j_s\gets\arg\max_{j\in\mathcal E_s}a_j(S_{s-1})\)
    \State \(S_s\gets S_{s-1}\cup\{j_s\}\)
\EndFor
\Ensure Selected synthetic set \(S_B\)
\end{algorithmic}

\vspace{3pt}
\hrule
\vspace{-0.8\baselineskip}
\end{wrapfigure}

Given a synthetic-data budget \(B\), TATC starts from \(S_0=\varnothing\) and adds one candidate at a time. At step \(s\), it considers the \(\min\{\kappa s,N\}\) highest-ranked candidates under \(h_j\), excluding selected examples, and selects the one with the largest coverage gain \(a_j(S_{s-1})\). Thus, the shortlist expands with \(s\), while the coverage score changes with the selected set. We use \(\kappa=2\) in all experiments. The procedure stops after \(B\) selections and returns \(S_B\) for subsequent fine-tuning. Implementation details are provided in Appendix~\ref{app:algorithm}.

\paragraph{Relation to the linear theory and scope.}
TATC carries the theory's distinction between input information and response error into nonlinear fine-tuning: \(a_j(S)\) measures target-weighted prediction-variance reduction, while \(h_j\) evaluates the local training effect of the complete example on a labeled real probe. Note that the two scores are not on a common scale: \(h_j\) defines the shortlist and \(a_j(S)\) selects within it. Thus, we provide a fixed-checkpoint approximation, where \(h_j\) is computed once at \(\bm\theta_0\) and \(a_j(S)\) changes with the selected set. See details in Appendix~\ref{app:nonlinear}.

\section{Experiments}
\label{sec:experiments}
\label{sec:math-sft}

\begin{figure}[t]
\centering
\includegraphics[width=.85\textwidth]{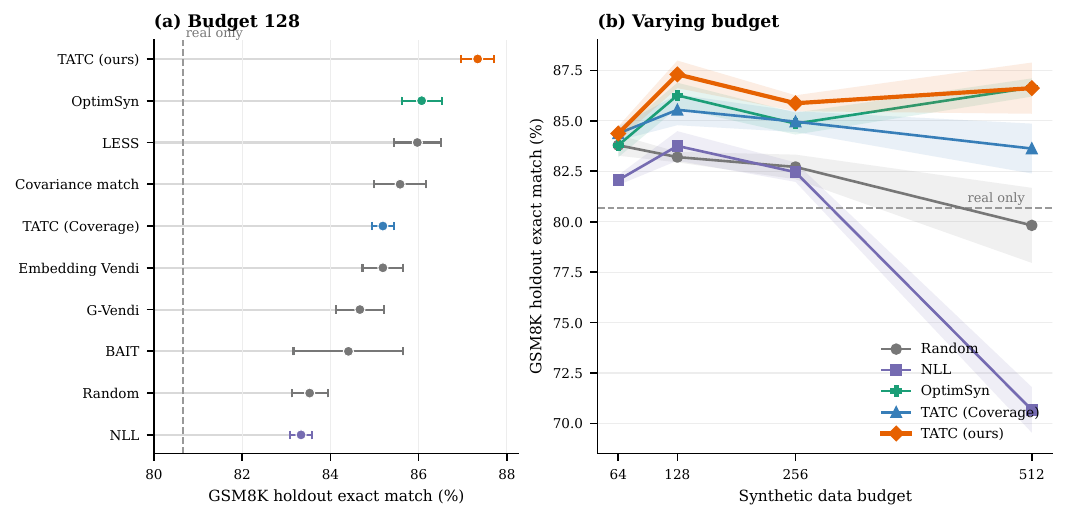}
\caption{
\textbf{(a)} Exact match at \(B=128\) on the GSM8K holdout. The dashed line marks the checkpoint trained on real data. \textbf{(b)} Exact match across selection budgets. Each method selects nested subsets, with each subset trained independently from the same checkpoint. Error bars show one standard deviation over three independent training runs.}
\label{fig:math-main}
\end{figure}

\paragraph{Experimental setup.}

We evaluate Qwen2.5-Math-1.5B-Instruct on
GSM8K. A LoRA adapter~\citep{hu2022lora} with rank 16 is trained for two epochs on 2,048 real examples, producing one checkpoint \(\bm\theta_0\) shared by every method. A disjoint set
of 512 labeled real examples serves as the target probe, and a further 1,024
examples form the holdout set used only for evaluation. The
candidate pool contains \(N=4{,}096\) deduplicated trajectories from
OpenR1-Math-220k~\citep{openr1math}, with at most one trajectory per source
problem; 65\% of the responses are verified as correct. Each selected subset is trained
for one epoch from \(\bm\theta_0\) with an effective batch size of 16.

For TATC, we average the hidden states from the final transformer layer over prompt tokens to obtain 1,536-dimensional representations, then reduce them to 64 dimensions using principal component analysis (PCA) fitted on the 2,048 real-training examples. The same projection is applied to the real probe and synthetic candidates.

\paragraph{Setup and baselines.}

We evaluate synthetic data selection at budget \(B\in\{64,128,256,512\}\), fine-tuning each selected subset from the same checkpoint and evaluating exact match on the GSM8K holdout. Baselines include Random, negative log-likelihood (NLL), Embedding Vendi~\citep{friedman2023vendi}, G-Vendi~\citep{jung2025prismatic}, LESS~\citep{xia2024less}, OptimSyn~\citep{fan2026optimsyn}, Covariance Matching~\citep{rezaei2026highdim}, and BAIT~\citep{ash2021bait}. We additionally include a variant of TATC that removes the training-effect component, TATC (Coverage), serving as an ablation to test whether target coverage alone is sufficient. Baseline implementations and further TATC details are provided in Appendices~\ref{app:baselines} and~\ref{app:math-details}.

\paragraph{Results.}

At \(B=128\), TATC achieves \(87.34\%\) exact match accuracy, exceeding the strongest competitor, OptimSyn (\(86.07\%\)), and the TATC (Coverage) ablation (\(85.19\%\)); the 2.15-point ablation gap demonstrates the added value of the training-effect component (Figure~\ref{fig:math-main}a). This advantage over coverage alone persists across larger budgets and reaches 2.99 points at \(B=512\). TATC is higher than OptimSyn at budgets 64, 128, and 256; at 512, their means are effectively tied within run-to-run variation (Figure~\ref{fig:math-main}b).
These results show that the benefit of the training-aware component emerges beyond the smallest budget and persists as more synthetic data are selected, while the gap between TATC and influence-based selection varies across budgets.

\section{Conclusion}
This work provides a novel perspective in understanding the value of synthetic data: the usefulness of an additional data point is not merely an intrinsic distribution property of this data point, but also depends on what the learner already has and what additional information this example can contribute. The linear model analysis makes this dependence explicit, showing that the same synthetic data can be valuable, redundant, or even harmful depending on the information and errors already present. This perspective also clarifies why continuously increasing the amount of synthetic data will eventually provide diminishing benefit, or even become detrimental.

Our findings also provide a practical basis for moving from understanding synthetic data to selecting data for LLM fine-tuning. TATC follows this insight by considering both how an example affects training and how it complements the data already available. Our experiments demonstrate its effectiveness for selecting synthetic mathematical reasoning data. More broadly, synthetic data pipelines should move beyond evaluating examples in isolation and consider their values in the context of the existing data and the downstream tasks for which they are used.

\label{sec:conclusion}

\section*{AI Use Statement}
Generative AI tools assisted with manuscript drafting and editing, literature
organization, code development and debugging, and preparation of figures and
tables. They were also used to review the exposition of theoretical arguments
and experimental methodology. The authors independently checked the mathematical
derivations, cited sources, experimental code and results. Generative AI output was not treated as evidence or independent
validation, and the authors take responsibility for the final content.

\section*{Ethics Statement}

This work uses public mathematical reasoning and standard classification
datasets and released model weights, and involves no human subjects or
intentionally collected private or personally identifying data. Generated
responses may contain systematic errors, and the proposed selection method
does not replace factual verification, safety evaluation, or appropriate data
governance. Our experiments do not establish the suitability of the method
for high-stakes or safety-critical applications.



\bibliography{iclr2027_conference}

@article{shumailov2024collapse,
  title        = {AI Models Collapse When Trained on Recursively Generated Data},
  author       = {Shumailov, Ilia and Shumaylov, Zakhar and Zhao, Yiren and Papernot, Nicolas and Anderson, Ross and Gal, Yarin},
  journal      = {Nature},
  volume       = {631},
  pages        = {755--759},
  year         = {2024}
}

@inproceedings{alemohammad2024mad,
  title        = {Self-Consuming Generative Models Go {MAD}},
  author       = {Alemohammad, Sina and Casco-Rodriguez, Josue and Luzi, Lorenzo and Humayun, Ahmed Imtiaz and Babaei, Hossein and LeJeune, Daniel and Siahkoohi, Ali and Baraniuk, Richard G.},
  booktitle    = {International Conference on Learning Representations},
  year         = {2024}
}

@inproceedings{yi2026verification,
  title        = {Escaping Model Collapse via Synthetic Data Verification: Near-Term Improvements and Long-Term Convergence},
  author       = {Yi, Bingji and Liu, Qiyuan and Cheng, Yuwei and Xu, Haifeng},
  booktitle    = {International Conference on Learning Representations},
  year         = {2026}
}

@inproceedings{rezaei2026highdim,
  title        = {High-Dimensional Analysis of Synthetic Data Selection},
  author       = {Rezaei, Parham and Kova{\v{c}}evi{\'c}, Filip and Locatello, Francesco and Mondelli, Marco},
  booktitle    = {International Conference on Learning Representations},
  year         = {2026}
}

@article{friedman2023vendi,
  title        = {The {Vendi Score}: A Diversity Evaluation Metric for Machine Learning},
  author       = {Friedman, Dan and Dieng, Adji Bousso},
  journal      = {Transactions on Machine Learning Research},
  year         = {2023}
}

@inproceedings{jung2025prismatic,
  title        = {Prismatic Synthesis: Gradient-based Data Diversification Boosts Generalization in {LLM} Reasoning},
  author       = {Jung, Jaehun and Han, Seungju and Lu, Ximing and Hallinan, Skyler and Acuna, David and Prabhumoye, Shrimai and Patwary, Mostofa and Shoeybi, Mohammad and Catanzaro, Bryan and Choi, Yejin},
  booktitle    = {Advances in Neural Information Processing Systems},
  volume       = {38},
  doi          = {10.52202/085713-3032},
  year         = {2025}
}

@inproceedings{ash2021bait,
  title        = {Gone Fishing: Neural Active Learning with Fisher Embeddings},
  author       = {Ash, Jordan T. and Goel, Surbhi and Krishnamurthy, Akshay and Kakade, Sham},
  booktitle    = {Advances in Neural Information Processing Systems},
  volume       = {34},
  pages        = {8927--8939},
  year         = {2021}
}

@inproceedings{chang2026spice,
  title        = {{SPICE}: Submodular Penalized Information--Conflict Selection for Efficient Large Language Model Training},
  author       = {Chang, Powei and Zhang, Jinpeng and Chen, Bowen and Wang, Chenyu and Guo, Chenlu and Zhang, Yixing and Gao, Yukang and Xiang, JianXiang and Gao, Yue and Sun, Chaoqun and Chen, Yiyi and Kong, Dongying},
  booktitle    = {International Conference on Learning Representations},
  year         = {2026}
}

@inproceedings{koh2017influence,
  title        = {Understanding Black-Box Predictions via Influence Functions},
  author       = {Koh, Pang Wei and Liang, Percy},
  booktitle    = {Proceedings of the 34th International Conference on Machine Learning},
  series       = {Proceedings of Machine Learning Research},
  volume       = {70},
  pages        = {1885--1894},
  year         = {2017}
}

@inproceedings{xia2024less,
  title        = {{LESS}: Selecting Influential Data for Targeted Instruction Tuning},
  author       = {Xia, Mengzhou and Malladi, Sadhika and Gururangan, Suchin and Arora, Sanjeev and Chen, Danqi},
  booktitle    = {Proceedings of the 41st International Conference on Machine Learning},
  series       = {Proceedings of Machine Learning Research},
  volume       = {235},
  pages        = {54104--54132},
  year         = {2024}
}

@inproceedings{fan2026optimsyn,
  title        = {{OptimSyn}: Influence-Guided Rubrics Optimization for Synthetic Data Generation},
  author       = {Fan, Zhiting and Chen, Ruizhe and Hu, Tianxiang and Peng, Ru and Huang, Zenan and Xu, Haokai and Chen, Yixin and Wu, Jian and Zhao, Junbo and Liu, Zuozhu},
  booktitle    = {International Conference on Learning Representations},
  year         = {2026}
}

@article{cobbe2021gsm8k,
  title        = {Training Verifiers to Solve Math Word Problems},
  author       = {Cobbe, Karl and Kosaraju, Vineet and Bavarian, Mohammad and Chen, Mark and Jun, Heewoo and Kaiser, Lukasz and Plappert, Matthias and Tworek, Jerry and Hilton, Jacob and Nakano, Reiichiro and Hesse, Christopher and Schulman, John},
  journal      = {arXiv preprint arXiv:2110.14168},
  year         = {2021}
}

@misc{openr1math,
  title        = {Open R1: A Fully Open Reproduction of {DeepSeek-R1}},
  author       = {{Hugging Face}},
  month        = jan,
  year         = {2025},
  url          = {https://github.com/huggingface/open-r1}
}

@inproceedings{hu2022lora,
  title        = {{LoRA}: Low-Rank Adaptation of Large Language Models},
  author       = {Hu, Edward J. and Shen, Yelong and Wallis, Phillip and Allen-Zhu, Zeyuan and Li, Yuanzhi and Wang, Shean and Wang, Lu and Chen, Weizhu},
  booktitle    = {International Conference on Learning Representations},
  year         = {2022}
}

@inproceedings{wang2023selfinstruct,
  title        = {Self-Instruct: Aligning Language Models with Self-Generated Instructions},
  author       = {Wang, Yizhong and Kordi, Yeganeh and Mishra, Swaroop and Liu, Alisa and Smith, Noah A. and Khashabi, Daniel and Hajishirzi, Hannaneh},
  booktitle    = {Proceedings of the 61st Annual Meeting of the Association for Computational Linguistics (Volume 1: Long Papers)},
  pages        = {13484--13508},
  year         = {2023},
  publisher    = {Association for Computational Linguistics},
  doi          = {10.18653/v1/2023.acl-long.754},
  url          = {https://aclanthology.org/2023.acl-long.754/}
}

@inproceedings{zelikman2022star,
  title        = {{STaR}: Bootstrapping Reasoning With Reasoning},
  author       = {Zelikman, Eric and Wu, Yuhuai and Mu, Jesse and Goodman, Noah D.},
  booktitle    = {Advances in Neural Information Processing Systems},
  volume       = {35},
  year         = {2022},
  doi          = {10.52202/068431-1126}
}

@article{gunasekar2023textbooks,
  title        = {Textbooks Are All You Need},
  author       = {Gunasekar, Suriya and Zhang, Yi and Aneja, Jyoti and Mendes, Caio C{\'e}sar Teodoro and Del Giorno, Allie and Gopi, Sivakanth and Javaheripi, Mojan and Kauffmann, Piero and de Rosa, Gustavo and Saarikivi, Olli and Salim, Adil and Shah, Shital and Behl, Harkirat Singh and Wang, Xin and Bubeck, S{\'e}bastien and Eldan, Ronen and Kalai, Adam Tauman and Lee, Yin Tat and Li, Yuanzhi},
  journal      = {arXiv preprint arXiv:2306.11644},
  year         = {2023},
  url          = {https://arxiv.org/abs/2306.11644}
}

@inproceedings{killamsetty2021glister,
  title        = {{GLISTER}: Generalization Based Data Subset Selection for Efficient and Robust Learning},
  author       = {Killamsetty, Krishnateja and Sivasubramanian, Durga and Ramakrishnan, Ganesh and Iyer, Rishabh},
  booktitle    = {Proceedings of the AAAI Conference on Artificial Intelligence},
  volume       = {35},
  pages        = {8110--8118},
  year         = {2021},
  doi          = {10.1609/aaai.v35i9.16988}
}

@inproceedings{mindermann2022rholoss,
  title        = {Prioritized Training on Points That Are Learnable, Worth Learning, and Not Yet Learnt},
  author       = {Mindermann, S{\"o}ren and Brauner, Jan M. and Razzak, Muhammed T. and Sharma, Mrinank and Kirsch, Andreas and Xu, Winnie and H{\"o}ltgen, Benedikt and Gomez, Aidan N. and Morisot, Adrien and Farquhar, Sebastian and Gal, Yarin},
  booktitle    = {Proceedings of the 39th International Conference on Machine Learning},
  series       = {Proceedings of Machine Learning Research},
  volume       = {162},
  pages        = {15630--15649},
  year         = {2022},
  url          = {https://proceedings.mlr.press/v162/mindermann22a.html}
}

@inproceedings{ghorbani2019datashapley,
  title        = {Data Shapley: Equitable Valuation of Data for Machine Learning},
  author       = {Ghorbani, Amirata and Zou, James},
  booktitle    = {Proceedings of the 36th International Conference on Machine Learning},
  series       = {Proceedings of Machine Learning Research},
  volume       = {97},
  pages        = {2242--2251},
  year         = {2019},
  url          = {https://proceedings.mlr.press/v97/ghorbani19c.html}
}

@inproceedings{liu2024deita,
  title        = {What Makes Good Data for Alignment? A Comprehensive Study of Automatic Data Selection in Instruction Tuning},
  author       = {Liu, Wei and Zeng, Weihao and He, Keqing and Jiang, Yong and He, Junxian},
  booktitle    = {International Conference on Learning Representations},
  year         = {2024},
  url          = {https://openreview.net/forum?id=BTKAeLqLMw}
}

@inproceedings{chen2024alpagasus,
  title        = {{AlpaGasus}: Training a Better Alpaca with Fewer Data},
  author       = {Chen, Lichang and Li, Shiyang and Yan, Jun and Wang, Hai and Gunaratna, Kalpa and Yadav, Vikas and Tang, Zheng and Srinivasan, Vijay and Zhou, Tianyi and Huang, Heng and Jin, Hongxia},
  booktitle    = {International Conference on Learning Representations},
  year         = {2024},
  url          = {https://openreview.net/forum?id=FdVXgSJhvz}
}

@inproceedings{loshchilov2019adamw,
  title        = {Decoupled Weight Decay Regularization},
  author       = {Loshchilov, Ilya and Hutter, Frank},
  booktitle    = {International Conference on Learning Representations},
  year         = {2019},
  url          = {https://openreview.net/forum?id=Bkg6RiCqY7}
}

@inproceedings{kunstner2019empiricalfisher,
  title        = {Limitations of the Empirical Fisher Approximation for Natural Gradient Descent},
  author       = {Kunstner, Frederik and Balles, Lukas and Hennig, Philipp},
  booktitle    = {Advances in Neural Information Processing Systems},
  volume       = {32},
  year         = {2019}
}

@techreport{krizhevsky2009cifar,
  title        = {Learning Multiple Layers of Features from Tiny Images},
  author       = {Krizhevsky, Alex},
  institution  = {University of Toronto},
  year         = {2009},
  url          = {https://www.cs.toronto.edu/~kriz/learning-features-2009-TR.pdf}
}

@inproceedings{pang2004sentimental,
  title        = {A Sentimental Education: Sentiment Analysis Using Subjectivity Summarization Based on Minimum Cuts},
  author       = {Pang, Bo and Lee, Lillian},
  booktitle    = {Proceedings of the 42nd Annual Meeting of the Association for Computational Linguistics ({ACL}-04)},
  pages        = {271--278},
  year         = {2004},
  doi          = {10.3115/1218955.1218990},
  url          = {https://aclanthology.org/P04-1035/}
}

@inproceedings{he2016resnet,
  title        = {Deep Residual Learning for Image Recognition},
  author       = {He, Kaiming and Zhang, Xiangyu and Ren, Shaoqing and Sun, Jian},
  booktitle    = {Proceedings of the IEEE Conference on Computer Vision and Pattern Recognition},
  pages        = {770--778},
  year         = {2016},
  doi          = {10.1109/CVPR.2016.90}
}

@article{yang2024qwenmath,
  title        = {{Qwen2.5-Math} Technical Report: Toward Mathematical Expert Model via Self-Improvement},
  author       = {Yang, An and Zhang, Beichen and Hui, Binyuan and Gao, Bofei and Yu, Bowen and Li, Chengpeng and Liu, Dayiheng and Tu, Jianhong and Zhou, Jingren and Lin, Junyang and Lu, Keming and Xue, Mingfeng and Lin, Runji and Liu, Tianyu and Ren, Xingzhang and Zhang, Zhenru},
  journal      = {arXiv preprint arXiv:2409.12122},
  year         = {2024},
  url          = {https://arxiv.org/abs/2409.12122}
}

@inproceedings{ba2024fillgaps,
  title        = {Fill In The Gaps: Model Calibration and Generalization with Synthetic Data},
  author       = {Ba, Yang and Mancenido, Michelle V. and Pan, Rong},
  booktitle    = {Proceedings of the 2024 Conference on Empirical Methods in Natural Language Processing},
  pages        = {17211--17225},
  year         = {2024},
  publisher    = {Association for Computational Linguistics},
  doi          = {10.18653/v1/2024.emnlp-main.955},
  url          = {https://aclanthology.org/2024.emnlp-main.955/}
}

@inproceedings{hsieh2023distilling,
  title        = {Distilling Step-by-Step! Outperforming Larger Language Models with Less Training Data and Smaller Model Sizes},
  author       = {Hsieh, Cheng-Yu and Li, Chun-Liang and Yeh, Chih-kuan and Nakhost, Hootan and Fujii, Yasuhisa and Ratner, Alex and Krishna, Ranjay and Lee, Chen-Yu and Pfister, Tomas},
  booktitle    = {Findings of the Association for Computational Linguistics: ACL 2023},
  pages        = {8003--8017},
  year         = {2023},
  publisher    = {Association for Computational Linguistics},
  doi          = {10.18653/v1/2023.findings-acl.507},
  url          = {https://aclanthology.org/2023.findings-acl.507/}
}

@inproceedings{wang2023synthesize,
  title        = {Let's Synthesize Step by Step: Iterative Dataset Synthesis with Large Language Models by Extrapolating Errors from Small Models},
  author       = {Wang, Ruida and Zhou, Wangchunshu and Sachan, Mrinmaya},
  booktitle    = {Findings of the Association for Computational Linguistics: EMNLP 2023},
  pages        = {11817--11831},
  year         = {2023},
  publisher    = {Association for Computational Linguistics},
  doi          = {10.18653/v1/2023.findings-emnlp.791},
  url          = {https://aclanthology.org/2023.findings-emnlp.791/}
}

@inproceedings{xu2024wizardlm,
  title        = {{WizardLM}: Empowering Large Pre-Trained Language Models to Follow Complex Instructions},
  author       = {Xu, Can and Sun, Qingfeng and Zheng, Kai and Geng, Xiubo and Zhao, Pu and Feng, Jiazhan and Tao, Chongyang and Lin, Qingwei and Jiang, Daxin},
  booktitle    = {International Conference on Learning Representations},
  year         = {2024},
  url          = {https://openreview.net/forum?id=CfXh93NDgH}
}

@inproceedings{yu2024metamath,
  title        = {{MetaMath}: Bootstrap Your Own Mathematical Questions for Large Language Models},
  author       = {Yu, Longhui and Jiang, Weisen and Shi, Han and Yu, Jincheng and Liu, Zhengying and Zhang, Yu and Kwok, James T. and Li, Zhenguo and Weller, Adrian and Liu, Weiyang},
  booktitle    = {International Conference on Learning Representations},
  year         = {2024},
  url          = {https://openreview.net/forum?id=N8N0hgNDRt}
}

@article{ge2024personahub,
  title        = {Scaling Synthetic Data Creation with 1,000,000,000 Personas},
  author       = {Ge, Tao and Chan, Xin and Wang, Xiaoyang and Yu, Dian and Mi, Haitao and Yu, Dong},
  journal      = {arXiv preprint arXiv:2406.20094},
  year         = {2024},
  url          = {https://arxiv.org/abs/2406.20094}
}

@inproceedings{xu2025magpie,
  title        = {Magpie: Alignment Data Synthesis from Scratch by Prompting Aligned {LLM}s with Nothing},
  author       = {Xu, Zhangchen and Jiang, Fengqing and Niu, Luyao and Deng, Yuntian and Poovendran, Radha and Choi, Yejin and Lin, Bill Yuchen},
  booktitle    = {International Conference on Learning Representations},
  year         = {2025},
  url          = {https://openreview.net/forum?id=Pnk7vMbznK}
}

@article{li2024glan,
  title        = {Synthetic Data (Almost) from Scratch: Generalized Instruction Tuning for Language Models},
  author       = {Li, Haoran and Dong, Qingxiu and Tang, Zhengyang and Wang, Chaojun and Zhang, Xingxing and Huang, Haoyang and Huang, Shaohan and Huang, Xiaolong and Huang, Zeqiang and Zhang, Dongdong and Gu, Yuxian and Cheng, Xin and Wang, Xun and Chen, Si-Qing and Dong, Li and Lu, Wei and Sui, Zhifang and Wang, Benyou and Lam, Wai and Wei, Furu},
  journal      = {arXiv preprint arXiv:2402.13064},
  year         = {2024},
  url          = {https://arxiv.org/abs/2402.13064}
}
\bibliographystyle{iclr2027_conference}

\appendix
\raggedbottom
\section{Mathematical Proofs in Linear Theory}
\label{app:proofs}

Let \(\bm X_r\in\mathbb R^{n\times d}\) and \(\bm X_s\in\mathbb R^{m\times d}\) be the real and synthetic feature matrices.  All feature and noise variables have finite second moments. The label model is
\begin{equation*}
 \bm y_r=\bm X_r\bm\theta^\star+\bm\varepsilon_r,
 \qquad
 \bm y_s=\bm X_s\bm\theta^\star+\bm d_s+\bm\varepsilon_s,
\end{equation*}
where, conditional on the feature matrices,
\begin{equation*}
 \mathbb E[\bm\varepsilon_r]=\bm0,
 \quad
 \mathbb E[\bm\varepsilon_s]=\bm0,
 \quad
 \operatorname{Cov}(\bm\varepsilon_r)=\bm\Omega_r,
 \quad
 \operatorname{Cov}(\bm\varepsilon_s)=\bm\Omega_s,
 \quad
 \operatorname{Cov}(\bm\varepsilon_r,\bm\varepsilon_s)=\bm0.
\end{equation*}
The systematic label error \(\bm d_s\) is nonrandom.  Selected synthetic subsets and candidate indices are fixed, or chosen independently of the realized noise innovations; expectations below are conditional on the relevant feature matrices. The least-squares estimator and downstream second-moment matrix are
\begin{align*}
 \widehat{\bm\theta}_s
 &=\arg\min_{\bm\theta}
 \left\{
 \tfrac12\|\bm X_r\bm\theta-\bm y_r\|_2^2
 +\tfrac12\|\bm X_s\bm\theta-\bm y_s\|_2^2
 \right\},\\
 \bm Q
 &=\mathbb E_{X\sim P_{\mathrm{test}}}
 [\bm\phi(X)\bm\phi(X)^\top]\succeq\bm0.
\end{align*}
For any parameter estimate \(\widehat{\bm\theta}\), its downstream risk is
\begin{equation*}
 \Risk_Q(\widehat{\bm\theta})
 =\|\widehat{\bm\theta}-\bm\theta^\star\|_{\bm Q}^2
 :=(\widehat{\bm\theta}-\bm\theta^\star)^\top
 \bm Q(\widehat{\bm\theta}-\bm\theta^\star).
\end{equation*}
Finally, define
\begin{equation*}
 \bm A=\bm X_r^\top\bm X_r+\bm X_s^\top\bm X_s,
 \qquad
 \bm u=\bm X_s^\top\bm d_s,
 \qquad
 \bm N=\bm X_r^\top\bm\Omega_r\bm X_r
       +\bm X_s^\top\bm\Omega_s\bm X_s.
\end{equation*}

\subsection{Proof of Proposition~\ref{prop:master-risk}}
\label{app:master-proof}

Because \(\bm A\) is nonsingular, the normal equations for the least-squares
problem above have the unique solution
\begin{align*}
 \widehat{\bm\theta}_s
 &=\bm A^{-1}(\bm X_r^\top\bm y_r+\bm X_s^\top\bm y_s)\\
 &=\bm\theta^\star
   +\bm A^{-1}\bm u
   +\bm A^{-1}
     (\bm X_r^\top\bm\varepsilon_r
      +\bm X_s^\top\bm\varepsilon_s).
\end{align*}
Write
\(
 \bm b=\bm A^{-1}\bm u
\)
and
\(
 \bm z=\bm A^{-1}
 (\bm X_r^\top\bm\varepsilon_r+\bm X_s^\top\bm\varepsilon_s).
\)
The conditional noise assumptions give
\begin{equation*}
 \mathbb E[\bm z\mid\bm X_r,\bm X_s]=\bm0,
 \qquad
 \operatorname{Cov}(\bm z\mid\bm X_r,\bm X_s)
 =\bm A^{-1}\bm N\bm A^{-1}.
\end{equation*}
Substituting \(\widehat{\bm\theta}_s-\bm\theta^\star=\bm b+\bm z\)
into the downstream-risk definition and taking the conditional expectation yields
\begin{align*}
 \mathbb E[\Risk_Q(\widehat{\bm\theta}_s)\mid\bm X_r,\bm X_s]
 &=\bm b^\top\bm Q\bm b
   +2\bm b^\top\bm Q\mathbb E[\bm z\mid\bm X_r,\bm X_s]
   +\mathbb E[\bm z^\top\bm Q\bm z\mid\bm X_r,\bm X_s]\\
 &=\bm b^\top\bm Q\bm b
   +\tr\!\left(
      \bm Q\operatorname{Cov}(\bm z\mid\bm X_r,\bm X_s)\right)\\
 &=(\bm A^{-1}\bm u)^\top\bm Q(\bm A^{-1}\bm u)
   +\tr(\bm Q\bm A^{-1}\bm N\bm A^{-1}).
\end{align*}
The cross term vanishes because
\(\mathbb E[\bm z\mid\bm X_r,\bm X_s]=\bm0\), and the last equality uses
\(\mathbb E[\bm z^\top\bm Q\bm z\mid\bm X_r,\bm X_s]
=\tr(\bm Q\operatorname{Cov}(\bm z\mid\bm X_r,\bm X_s))\) for a conditionally zero-mean vector.  This proves Equation~\eqref{eq:master-risk}. \(\square\)

\subsection{Proof of Corollary~\ref{cor:set-value}}
\label{app:set-value-proof}

For a selected set \(S\), write
\begin{equation*}
 \bm A_r=\bm X_r^\top\bm X_r,
 \qquad
 \bm A_S=\bm X_S^\top\bm X_S,
 \qquad
 \bm A=\bm A_r+\bm A_S.
\end{equation*}
Under \(\bm\Omega_r=\sigma^2\bm I\) and
\(\bm\Omega_S=(\sigma^2+\tau^2)\bm I\), the noise matrix in
the downstream-risk decomposition becomes
\begin{equation*}
 \bm N
 =\sigma^2\bm A_r+(\sigma^2+\tau^2)\bm A_S
 =\sigma^2\bm A+\tau^2\bm A_S.
\end{equation*}
The real-only estimator is unbiased and, because \(\bm A_r\succ\bm0\), its
conditional risk is
\begin{equation*}
 \mathbb E[\Risk_Q(\widehat{\bm\theta}_r)\mid\bm X_r,\bm X_S]
 =\sigma^2\tr(\bm Q\bm A_r^{-1}).
\end{equation*}
Proposition~\ref{prop:master-risk}, the expression for \(\bm N\) above, and
\(\bm u=\bm X_S^\top\bm d_S\) give the selected-set risk
\begin{align*}
 \mathbb E[\Risk_Q(\widehat{\bm\theta}_S)\mid\bm X_r,\bm X_S]
 ={}&
 (\bm A^{-1}\bm X_S^\top\bm d_S)^\top
 \bm Q(\bm A^{-1}\bm X_S^\top\bm d_S)\\
 &+\sigma^2\tr(\bm Q\bm A^{-1})
 +\tau^2\tr
 (\bm Q\bm A^{-1}\bm A_S\bm A^{-1}).
\end{align*}
Subtracting the selected-set risk from the real-only risk gives
Equation~\eqref{eq:set-value}.
\(\square\)

\subsection{Proof of Proposition~\ref{prop:modes}}
\label{app:modes-proof}

Let
\begin{equation*}
 \bm B=\bm A_r^{-1/2}\bm A_S\bm A_r^{-1/2}
 =\bm R\bm M\bm R^\top,
 \quad
 \bm M=\operatorname{diag}(\mu_1,\ldots,\mu_d),
 \quad
 \bm D=\bm A_r^{-1/2}\bm Q\bm A_r^{-1/2}.
\end{equation*}
Then
\begin{equation*}
 \bm A
 =\bm A_r^{1/2}(\bm I+\bm B)\bm A_r^{1/2},
 \qquad
 \bm A^{-1}
 =\bm A_r^{-1/2}(\bm I+\bm B)^{-1}\bm A_r^{-1/2}.
\end{equation*}
Using cyclic invariance of the trace,
\begin{align*}
 \tr[\bm Q(\bm A_r^{-1}-\bm A^{-1})]
 &=\tr\!\left[
 \bm D\{\bm I-(\bm I+\bm B)^{-1}\}\right]\\
 &=\tr\!\left[
 \bm R^\top\bm D\bm R
 \operatorname{diag}\!\left(
 \frac{\mu_i}{1+\mu_i}
 \right)_{i=1}^d\right]\\
 &=\sum_{i=1}^d
 (\bm r_i^\top\bm D\bm r_i)
 \frac{\mu_i}{1+\mu_i}.
\end{align*}
Thus the information term equals
\(\sigma^2\sum_i w_i\mu_i/(1+\mu_i)\), where
\(w_i=\bm r_i^\top\bm D\bm r_i\).

The excess-noise term is handled in the same basis:
\begin{align*}
 \tr(\bm Q\bm A^{-1}\bm A_S\bm A^{-1})
 &=\tr\!\left[
 \bm D(\bm I+\bm B)^{-1}\bm B(\bm I+\bm B)^{-1}
 \right]\\
 &=\sum_{i=1}^d w_i\frac{\mu_i}{(1+\mu_i)^2}.
\end{align*}
These two trace decompositions require no
alignment between \(\bm D\) and \(\bm B\), because the matrices multiplying
\(\bm R^\top\bm D\bm R\) are diagonal and a trace uses only its diagonal.

It remains to evaluate systematic error.  Suppose
\(\bm d_S=\bm X_S\bm\delta\) and define
\(\bm\eta=\bm R^\top\bm A_r^{1/2}\bm\delta\).  Then
\begin{align*}
 \bm A^{-1}\bm X_S^\top\bm d_S
 &=\bm A^{-1}\bm A_S\bm\delta\\
 &=\bm A_r^{-1/2}
 \bm R\operatorname{diag}\!\left(
   \frac{\mu_i}{1+\mu_i}
   \right)_{i=1}^d\bm\eta.
\end{align*}
Let \(\bm T=\operatorname{diag}(\mu_i/(1+\mu_i))\).  The squared
\(\bm Q\)-norm of this vector is
\begin{equation*}
 \bm\eta^\top\bm T\bm R^\top\bm D\bm R\bm T\bm\eta.
\end{equation*}
When \(\bm R^\top\bm D\bm R\) is diagonal, its \(i\)-th diagonal entry is
\(w_i\), so this expression reduces to
\begin{equation*}
 \sum_{i=1}^d
 w_i\eta_i^2\left(\frac{\mu_i}{1+\mu_i}\right)^2.
\end{equation*}
Combining the information, excess-noise, and systematic-error expressions in
Corollary~\ref{cor:set-value} proves
Equation~\eqref{eq:mode-value}. \(\square\)

\subsection{Proof of Corollary~\ref{cor:beneficial-direction}}
\label{app:beneficial-proof}

For \(\mu_i>0\) and \(w_i>0\), the sign of the \(i\)-th contribution in
Equation~\eqref{eq:mode-value} is the sign of
\begin{equation*}
 \sigma^2\frac{\mu_i}{1+\mu_i}
 -\tau^2\frac{\mu_i}{(1+\mu_i)^2}
 -\eta_i^2\frac{\mu_i^2}{(1+\mu_i)^2}.
\end{equation*}
Multiplying by the positive quantity
\((1+\mu_i)^2/\mu_i\) shows that this contribution is positive exactly when
\begin{equation*}
 \sigma^2(1+\mu_i)-\tau^2-\eta_i^2\mu_i>0.
\end{equation*}
Rearranging and dividing by \(\mu_i>0\) gives
\begin{equation*}
 \eta_i^2<\sigma^2+\frac{\sigma^2-\tau^2}{\mu_i},
\end{equation*}
which proves the corollary. \(\square\)

\subsection{Proof of Corollary~\ref{cor:optimal-amount}}
\label{app:amount-proof}

Hold \(w_i\) and \(\eta_i\) unchanged while varying \(\mu_i\ge0\).
For \(w_i>0\), it is enough to maximize the bracketed contribution
\begin{equation*}
 f(\mu)
 =\sigma^2\frac{\mu}{1+\mu}
  -\tau^2\frac{\mu}{(1+\mu)^2}
  -\eta_i^2\left(\frac{\mu}{1+\mu}\right)^2.
\end{equation*}
Direct differentiation gives
\begin{equation*}
 f'(\mu)
 =\frac{(\sigma^2-\tau^2)
 +\mu(\sigma^2+\tau^2-2\eta_i^2)}{(1+\mu)^3}.
\end{equation*}
The denominator is positive.  If \(\sigma^2>\tau^2\) and
\(2\eta_i^2>\sigma^2+\tau^2\), the numerator starts positive and decreases
linearly, crossing zero once at
\begin{equation*}
 \mu_i^\star
 =\frac{\sigma^2-\tau^2}
 {2\eta_i^2-\sigma^2-\tau^2}.
\end{equation*}
Consequently, \(f\) increases before this point and decreases after it, so the
point is the unique global maximizer on \([0,\infty)\).

For the remaining cases, write \(s=\sigma^2\), \(t=\tau^2\), and
\(r=\eta_i^2\).  Equivalently,
\(f(\mu)=\mu[(s-t)+\mu(s-r)]/(1+\mu)^2\), with
\(f(0)=0\) and \(\lim_{\mu\to\infty}f(\mu)=s-r\).
If \(s<t\) and \(r\ge s\), or if \(s=t\) and \(r>s\), this numerator is
strictly negative for \(\mu>0\), so zero is the unique maximizer.  If
\(s=t=r\), then \(f\equiv0\).  If \(s>t\) and
\(2r\le s+t\), the derivative above is positive for
every finite \(\mu\), so the supremum is approached at infinity.  If
\(s\le t\) and \(r<s\), the derivative is positive for \(\mu>0\) when
\(s=t\); when \(s<t\), it changes once from negative to positive, after which
\(f\) increases to the positive limit \(s-r\).  Thus this case also has no
finite maximizer.  These conclusions concern the actual directional
contribution when \(w_i>0\); if \(w_i=0\), that contribution is identically
zero and every amount is optimal. \(\square\)

\subsection{Proof of Proposition~\ref{prop:nonidentifiability}}
\label{app:nonidentifiability-proof}

Take a direction \(\bm v\) satisfying the conditions in the proposition.  For
any scalar \(c\), define a second parameterization
\begin{equation*}
 \widetilde{\bm\theta}^{\star}=\bm\theta^\star+c\bm v,
 \qquad
 \widetilde{\bm d}_S=\bm d_S-c\bm X_S\bm v.
\end{equation*}
The real conditional mean is unchanged because
\begin{equation*}
 \bm X_r\widetilde{\bm\theta}^{\star}
 =\bm X_r\bm\theta^\star+c\bm X_r\bm v
 =\bm X_r\bm\theta^\star.
\end{equation*}
The synthetic conditional mean is also unchanged:
\begin{equation*}
 \bm X_S\widetilde{\bm\theta}^{\star}+\widetilde{\bm d}_S
 =\bm X_S(\bm\theta^\star+c\bm v)
  +\bm d_S-c\bm X_S\bm v
 =\bm X_S\bm\theta^\star+\bm d_S.
\end{equation*}
The noise laws are identical, so the two parameterizations induce exactly the
same distribution for all observed real and synthetic labels.  Nevertheless,
they assign different target conditional means in direction \(\bm v\), which matters
downstream because \(\bm v^\top\bm Q\bm v>0\).

For completeness, an admissible one-direction instance shows that this
ambiguity can reverse the sign of synthetic-data value.  Take the model space
to be the span of \(\bm v\), let \(\bm X_r\bm v=\bm0\),
\(\bm X_S\bm v\ne\bm0\), and set both noise covariances to zero.  Define the
real-only estimate by the Moore--Penrose, or equivalently minimum-norm, rule;
its coefficient in this unobserved direction is zero.  Since the scalar
selected-set Gram matrix is positive, the selected-set least-squares estimate
is unique.  Give the observed synthetic label the conditional mean
\(a\bm X_S\bm v\), with \(a\ne0\).  The two observationally equivalent worlds
are
\begin{align*}
 \mathcal W_+: \quad&
 \bm\theta^\star=a\bm v,
 \quad \bm d_S=\bm0,\\
 \mathcal W_-: \quad&
 \bm\theta^\star=-a\bm v,
 \quad \bm d_S=2a\bm X_S\bm v.
\end{align*}
They have the same real and synthetic conditional means.  Omitting common
orthogonal components and random-noise terms, the real-only and selected-set
risks in \(\mathcal W_+\) are respectively
\(a^2\bm v^\top\bm Q\bm v\) and zero, so the value is positive.  In
\(\mathcal W_-\), the corresponding risks are
\(a^2\bm v^\top\bm Q\bm v\) and
\(4a^2\bm v^\top\bm Q\bm v\), so the value is
\(-3a^2\bm v^\top\bm Q\bm v<0\).  No statistic of the observed data can know
which world generated them.

The sign-reversal construction is existential and uses the additional conditions---a
declared minimum-norm real-only rule, a nonsingular combined Gram matrix in the
one-direction instance, and zero random noise---that make \(\Value(S)\)
unambiguous.  The observational-equivalence result itself needs none of these
extra conditions.  Observational equivalence alone does not assert a
sign reversal for every fixed observed distribution. The proposition is complementary to the directional value
results, which assume \(\bm A_r\succ\bm0\); it is not an application of
Equation~\eqref{eq:set-value} with an ordinary inverse.  Its
information-theoretic conclusion is that without labeled real evidence in a
downstream-relevant direction, target signal and systematic label error
cannot be separated, and admissible worlds can assign opposite signs to value.
\(\square\)

\subsection{Proof of Theorem~\ref{thm:rank-one-marginal}}
\label{app:rank-one-proof}

The scalar variance \(\omega_j\) in the theorem represents a fresh, conditionally zero-mean noise
innovation.  Formally, the rank-one update below requires this innovation to be
uncorrelated, conditional on all feature matrices, with the real and already
selected label noises.  It also requires \(\bm A\) to be positive definite,
which makes \(c_j>0\).  If a candidate noise is correlated with existing
noises, the augmented \(\bm N\) contains cross-covariance terms and
Equation~\eqref{eq:rank-one-marginal} needs corresponding additional terms.

For the real data together with the current selected set \(S\), let
\begin{equation*}
 \bm b=\bm A^{-1}\bm u,
 \qquad
 \bm p_j=\bm A^{-1}\bm x_j,
 \qquad
 c_j=1+\bm x_j^\top\bm p_j,
 \qquad
 e_j=\delta_j-\bm x_j^\top\bm b,
 \qquad
 q_j=\bm p_j^\top\bm Q\bm p_j.
\end{equation*}
Adding candidate \(j\) changes the three sufficient statistics to
\begin{equation*}
 \bm A_+=\bm A+\bm x_j\bm x_j^\top,
 \qquad
 \bm u_+=\bm u+\bm x_j\delta_j,
 \qquad
 \bm N_+=\bm N+\omega_j\bm x_j\bm x_j^\top.
\end{equation*}
The Sherman--Morrison identity gives
\begin{equation*}
 \bm A_+^{-1}
 =\bm A^{-1}-\frac{\bm p_j\bm p_j^\top}{c_j},
 \qquad
 \bm A_+^{-1}\bm x_j=\frac{\bm p_j}{c_j}.
\end{equation*}
Using \(\bm A^{-1}\bm u=\bm b\) and
\(\bm p_j^\top\bm u=\bm x_j^\top\bm b\), the new bias vector is
\begin{align*}
 \bm b_+
 &=\bm A_+^{-1}\bm u_+\\
 &=\left(\bm A^{-1}-\frac{\bm p_j\bm p_j^\top}{c_j}\right)
   (\bm u+\bm x_j\delta_j)\\
 &=\bm b+\frac{\delta_j-\bm x_j^\top\bm b}{c_j}\bm p_j
 =\bm b+\frac{e_j}{c_j}\bm p_j.
\end{align*}
Hence the reduction in the squared-bias component of risk is
\begin{align*}
 \bm b^\top\bm Q\bm b-\bm b_+^\top\bm Q\bm b_+
 =-\frac{2e_j\bm b^\top\bm Q\bm p_j}{c_j}
  -\frac{e_j^2q_j}{c_j^2}.
\end{align*}

For the variance component, first expand the contribution from the existing
noise matrix.  Substituting the Sherman--Morrison update and using cyclic
invariance of the trace gives
\begin{align*}
 &\tr(\bm Q\bm A^{-1}\bm N\bm A^{-1})
 -\tr(\bm Q\bm A_+^{-1}\bm N\bm A_+^{-1})\\
 &\qquad=
 \frac{2\bm p_j^\top\bm Q\bm A^{-1}\bm N\bm p_j}{c_j}
 -\frac{q_j\bm p_j^\top\bm N\bm p_j}{c_j^2}.
\end{align*}
The candidate's own random-noise contribution to the new risk is
\begin{align*}
 \omega_j\tr
 (\bm Q\bm A_+^{-1}\bm x_j\bm x_j^\top\bm A_+^{-1})
 =\frac{\omega_jq_j}{c_j^2}.
\end{align*}
Because marginal value is old risk minus new risk, adding the squared-bias
change and the two variance changes above
proves Equation~\eqref{eq:rank-one-marginal}.

To verify the clean homoscedastic specialization in
Equation~\eqref{eq:coverage-marginal}, suppose
\(\bm N=\sigma^2\bm A\), \(\bm b=\bm0\), \(\delta_j=0\), and
\(\omega_j=\sigma^2\).  Then
\begin{equation*}
 \bm p_j^\top\bm Q\bm A^{-1}\bm N\bm p_j=\sigma^2q_j,
 \qquad
 \bm p_j^\top\bm N\bm p_j
 =\sigma^2\bm x_j^\top\bm A^{-1}\bm x_j
 =\sigma^2(c_j-1).
\end{equation*}
Substitution into Equation~\eqref{eq:rank-one-marginal} leaves
\begin{equation*}
 \Delta_j(S)
 =\frac{2\sigma^2q_j}{c_j}
 -\frac{\sigma^2q_j(c_j-1+1)}{c_j^2}
 =\frac{\sigma^2q_j}{c_j}.
\end{equation*}
Thus the stated coverage marginal follows exactly.

\paragraph{More general specializations.}
Retain \(\bm N=\sigma^2\bm A\) and \(\bm b=\bm0\), so that \(e_j=\delta_j\),
but now let the candidate's systematic error \(\delta_j\) and noise variance
\(\omega_j\) be arbitrary rather than zero and \(\sigma^2\) respectively. As
computed above,
\begin{equation*}
 \bm p_j^\top\bm Q\bm A^{-1}\bm N\bm p_j=\sigma^2q_j,
 \qquad
 \bm p_j^\top\bm N\bm p_j=\sigma^2(c_j-1).
\end{equation*}
Substituting into Equation~\eqref{eq:rank-one-marginal}, the two variance
terms combine as
\begin{equation*}
 \frac{2\sigma^2q_j}{c_j}
 -\frac{q_j\big(\sigma^2(c_j-1)+\omega_j\big)}{c_j^2}
 =\frac{q_j\big(\sigma^2c_j+\sigma^2-\omega_j\big)}{c_j^2},
\end{equation*}
and the two bias terms reduce to \(-\delta_j^2q_j/c_j^2\) since \(\bm b=\bm0\).
Adding the two gives the general specialization
\begin{equation*}
 \Delta_j(S)
 =\frac{q_j}{c_j^2}
 \Big(\sigma^2c_j+\sigma^2-\omega_j-\delta_j^2\Big).
\end{equation*}
Three cases follow by substitution. A candidate carrying only systematic
error (\(\omega_j=\sigma^2\)) contributes
\(\frac{q_j}{c_j^2}(\sigma^2c_j-\delta_j^2)\); a candidate carrying only the
excess noise of Corollary~\ref{cor:set-value}
(\(\delta_j=0,\ \omega_j=\sigma^2+\tau^2\)) contributes
\(\frac{q_j}{c_j^2}(\sigma^2c_j-\tau^2)\); and a candidate carrying both
(\(\delta_j\neq0,\ \omega_j=\sigma^2+\tau^2\)) contributes
\(\frac{q_j}{c_j^2}(\sigma^2c_j-\tau^2-\delta_j^2)\), mirroring the
information-gain/noise-cost/bias-cost tradeoff of
Corollary~\ref{cor:set-value} at the level of a single candidate. Setting
\(\delta_j=0\) and \(\omega_j=\sigma^2\) in the first case recovers
Equation~\eqref{eq:coverage-marginal}. \(\square\)

\section{Controlled Tests of the Linear Theory}
\label{app:exact-tests}

Figure~\ref{fig:exact-regime} uses fixed representations of real text and image
inputs paired with controlled targets.  Panel (a) compares the exact rank-one
marginal with before--after least-squares fits under shared existing-data noise
and fresh candidate noise.  Panel
(b) uses labeled probe data to select a pseudo-labeled-to-real ratio and
compares that choice with the held-out test optimum.

In both panels, \emph{real} identifies the source of an input, whereas
\emph{synthetic} identifies its training target.  Panel (a) assigns controlled
scalar targets to real text inputs; panel (b) assigns labels predicted by a
separate linear model to real images.  Qwen2.5 and ResNet-50 provide frozen
representations, and all downstream models are ordinary least squares.
Holding the candidate inputs fixed allows label error to change without also
changing the information provided by the inputs.

\subsection{Candidate value in text data}

\paragraph{Data and representation.}
We use the 8,000-example training portion of the \texttt{SetFit/subj}
subjectivity-classification dataset.  Qwen2.5-0.5B is frozen and used only as a
feature extractor: each text is truncated to 96 tokens, the final-layer hidden
states of its nonpadding tokens are averaged, and the resulting representation
has 896 coordinates.

For each of three random partitions, examples are assigned without replacement
to the five disjoint roles in Table~\ref{tab:text-partitions}. PCA is fitted
only on \(D_{\mathrm{rep}}\), retains 64 components, and is applied unchanged
to every other role; an appended intercept gives \(\bm x\in\mathbb R^{65}\).

\begin{table}[H]
\centering
\small
\caption{Data partitions for the text experiment.}
\label{tab:text-partitions}
\begin{tabular}{@{}lrp{.68\linewidth}@{}}
\toprule
Partition & Examples & Role \\
\midrule
\(D_{\mathrm{rep}}\) & 1,500 & Fit PCA and the reference linear target \\
\(D_{\mathrm{cal}}\) & 500 & Estimate the residual variance scale \\
\(D_r\) & 1,200 & Define the real training information \\
\(D_q\) & 1,000 & Estimate the target moment \(\widehat{\bm Q}\) \\
\(\mathcal P_s\) & 3,800 & Supply candidate inputs \\
\bottomrule
\end{tabular}
\end{table}

We map the binary labels to \(y^\pm=2y-1\), fit the reference linear target
\(\widehat{\bm\theta}^{\star}\) on \(D_{\mathrm{rep}}\), and use its residual
MSE on the disjoint \(D_{\mathrm{cal}}\) set to define
\(\widehat\sigma^2\). This scale includes both residual variation and linear
approximation error. Observed labels in the candidate pool are not used as
candidate targets.

\paragraph{Controlled target construction.}
The 1,200 real inputs are always present.  We sample a current set \(S\) of
size \(32,64,128,256,512,\) or \(768\), together with one candidate
\(j\notin S\), from the 3,800-input pool.  The real training targets follow
\(\bm y_r=\bm X_r\widehat{\bm\theta}^{\star}+\bm\varepsilon_r\), with
zero-mean variance \(\widehat\sigma^2\).

To construct systematic label error, we set the intercept coordinate of
\(\widehat{\bm\theta}^{\star}\) to zero and normalize the resulting direction
\(\bm g\) so that \(\bm g^\top\widehat{\bm Q}\bm g=1\).  For
\(a\in\{0.5,1,2,4,8\}\), define
\[
 \bm\gamma=a\sqrt{\widehat\sigma^2}\bm g,
 \qquad
 \delta_j=\bm x_j^\top\bm\gamma.
\]
Removing the intercept and normalizing under \(\widehat{\bm Q}\) makes the
systematic error feature-dependent, task-relevant, and comparable across
partitions.  The two matched contexts use
\[
 \bm d_S^{(-)}=-\bm X_S\bm\gamma,\quad
 \bm u_-=-\bm A_S\bm\gamma,
 \qquad
 \bm d_S^{(+)}=+\bm X_S\bm\gamma,\quad
 \bm u_+=+\bm A_S\bm\gamma.
\]
Their controlled training targets are
\(\bm y_S^{(\pm)}=\bm X_S\widehat{\bm\theta}^{\star}
+\bm d_S^{(\pm)}+\bm\varepsilon_S\) and
\(y_j=\bm x_j^\top\widehat{\bm\theta}^{\star}+\delta_j+\xi_j\).
A matched pair has the same \(\bm X_r,\bm X_S,\bm x_j,\delta_j\), and therefore
the same coverage score; only the sign of the error accumulated in \(S\)
changes.  The candidate can therefore cancel existing bias in the minus
context and reinforce it in the plus context, directly varying
\(e_j=\delta_j-\bm x_j^\top\bm A^{-1}\bm u\) in
Theorem~\ref{thm:rank-one-marginal}.

The matched contexts use the same random-noise variance (\(\tau^2=0\)), so
their difference comes only from systematic label error. For each partition,
set size, and error magnitude, we sample four input sets and evaluate both
signs, yielding 360 matched pairs.

\paragraph{Evaluation and results.}
For each context, the input-only comparator is the clean coverage marginal
\[
 s_{\mathrm{cov}}(j\mid S)
 =\widehat\sigma^2
 \frac{\bm p_j^\top\widehat{\bm Q}\bm p_j}
 {1+\bm x_j^\top\bm p_j},
 \qquad \bm p_j=\bm A^{-1}\bm x_j,
\]
and the error-aware score is Equation~\eqref{eq:rank-one-marginal}.
Independently of either score, we solve the least-squares model before and
after adding \(j\), using shared existing-data noise and fresh candidate noise,
and record the reduction in \(\widehat{\bm Q}\)-weighted risk.  Each context
averages 3,000 independent label-noise trials.

Coverage alone barely distinguishes beneficial from harmful candidates,
whereas the complete marginal predicts both the sign and magnitude of the
observed risk reduction (Table~\ref{tab:text-marginal-results}).

\begin{table}[H]
\centering
\small
\caption{Agreement with observed risk reductions in the text experiment.}
\label{tab:text-marginal-results}
\begin{tabular}{@{}llr@{}}
\toprule
Score & Metric & Value \\
\midrule
Coverage & AUROC & 0.519 \\
Complete marginal & AUROC & 0.99 \\
Complete marginal & Correlation & 0.99 \\
Complete marginal & Sign agreement & 99.2\% \\
\bottomrule
\end{tabular}
\end{table}

\subsection{Synthetic-data amount in image data}

\paragraph{Data and representation.}
We use CIFAR-10 images~\citep{krizhevsky2009cifar} represented by frozen
ResNet-50 penultimate-layer features~\citep{he2016resnet}.  Images are resized
to \(224\times224\), normalized with ImageNet channel statistics, and passed
through an ImageNet-1K-V2 pretrained ResNet-50, producing 2,048-dimensional
features.

For each of four random partitions, the 50,000 training images are assigned
without replacement to \(D_r\) (3,000 real-training images), \(D_q\) (3,000
labeled probe images), \(D_g\) (12,000 images used to fit label models),
\(\mathcal P_s\) (24,000 candidate images), and 8,000 unused images.  The
official 10,000-image test split remains separate.  PCA is fitted only on
\(D_r\), retains 128 components, and is applied unchanged to
\(D_q,D_g,\mathcal P_s\), and the test set.  At each amount, an unregularized
multiclass least-squares model is fitted jointly to the real and pseudo-labeled
examples.

\paragraph{Pseudo-label reliability.}
Each label model assigns its argmax class as a deterministic pseudo-label.
We vary reliability using five models fitted on \(D_g\): three use all
12,000 labeled examples or the first 400 or 150 examples, and two use all
12,000 examples after removing four or 12 task-relevant PCA coordinates.
Figure~\ref{fig:exact-regime}(b) displays the model fitted on all examples,
the model fitted on 400 examples, and the model with 12 coordinates removed.
No candidate or test labels are used to fit these models.
For the coordinate-removed condition, directions are ranked by explained
variance weighted by their reference-classifier coefficients, concentrating
errors in task-relevant directions.  Conditional on the fitted model,
pseudo-labels are deterministic, so their errors enter the systematic-error
term rather than adding random label noise.

For each pseudo-label model, the downstream model is fitted at
pseudo-labeled-to-real ratios
\[
 0,\ 0.05,\ 0.10,\ 0.20,\ 0.40,\ 0.80,\ 1.00,\ 1.50,\ 2.00,\ 3.00,\ 4.00,\ 6.00,\ 8.00.
\]
All conditions use the same candidate order, and larger amounts extend smaller
subsets.  The black cross in Figure~\ref{fig:exact-regime}(b) marks the
test-MSE-minimizing grid point, and the gray interval contains ratios within
1\% of that minimum.  The test split is used only for this retrospective
reference.

\paragraph{Probe-based risk estimate.}
The labeled probe set is divided into two disjoint halves.  Let
\(\widehat{\bm\beta}_a\) and \(\widehat{\bm\beta}_b\) be 128-by-10
least-squares class-score models fitted to the two 1,500-example halves.
The target moment \(\widehat{\bm Q}\) uses all 3,000 probe inputs, and
\(\widehat\sigma^2\) is the average per-coordinate residual MSE with the
\(n/(n-d)\) correction, where \(n=1{,}500\) and \(d=128\).  For a proposed
pseudo-labeled set, define
\[
 \bm B_S=\bm X_S^\top\bm Y_S,\qquad
 \bm U_a=\bm B_S-\bm A_S\widehat{\bm\beta}_a,\qquad
 \bm U_b=\bm B_S-\bm A_S\widehat{\bm\beta}_b.
\]
With \(\bm A=\bm A_r+\bm A_S\), the estimated systematic-error contribution is
\[
 \widehat R_{\mathrm{bias}}
 =\max\left\{0,
 \operatorname{tr}\!\left[
 (\bm A^{-1}\bm U_a)^\top
 \widehat{\bm Q}
 (\bm A^{-1}\bm U_b)
 \right]\right\},
\]
and the random-noise contribution is
\[
 \widehat R_{\mathrm{var}}
 =10\widehat\sigma^2
 \operatorname{tr}
 (\widehat{\bm Q}\bm A^{-1}\bm A_r\bm A^{-1}).
\]
The cross product avoids the positive inflation from squaring a single noisy
probe estimate; truncation enforces nonnegativity.
The variance term assumes isotropic residual variation in the trusted real
labels; deterministic pseudo-labels add no conditional random-noise
covariance.  We select the ratio minimizing
\(\widehat R_{\mathrm{bias}}+\widehat R_{\mathrm{var}}\).  Dropping the bias
term gives the input-only comparator, which selects the full 24,000-candidate
pool.

\paragraph{Amount-selection results.}
The probe estimate tracks the decrease in useful data amount as label
reliability falls (Table~\ref{tab:image-amount-results}). Across all five
conditions, predicted and test-optimal amounts have the same ordering
(Kendall \(\tau=1.00\)), and every prediction across the four partitions
falls within the 1\% test-optimal plateau. For the most reliable label model,
the test optimum on the evaluated grid occurs at the largest ratio.

\begin{table}[H]
\centering
\small
\caption{Mean label accuracy and ratios of synthetic to real data for the
three conditions in Figure~\ref{fig:exact-regime}(b).}
\label{tab:image-amount-results}
\begin{tabular}{@{}lrrr@{}}
\toprule
Label model & Label accuracy & Predicted ratio & Test optimum ratio \\
\midrule
12,000 training examples & 0.889 & 7.0 & 8.0 \\
400 training examples & 0.836 & 0.90 & 0.95 \\
12 coordinates removed & 0.604 & 0 & 0 \\
\bottomrule
\end{tabular}
\end{table}

This experiment evaluates amount selection when multiple feature directions
contribute, complementing the characterization of a single direction in
Corollary~\ref{cor:optimal-amount}.

\section{Baseline Implementations}
\label{app:baselines}

\subsection{Representations and scores used by the selectors}

\paragraph{Prompt features.}
For every real, probe, and candidate prompt, we average the
1,536-dimensional final-layer hidden states across prompt tokens to obtain one
prompt representation. Principal-component analysis is fitted on the
real-training representations and projects all partitions to 64 dimensions.
The same projection is used by Embedding Vendi, Covariance Matching,
TATC (Coverage), and the coverage stage of TATC.

\paragraph{Gradient features.}
For every complete example, we compute the gradient of the response loss with
respect to the trainable LoRA parameters. To make per-example storage and
matrix operations feasible, the gradient is mapped by a signed CountSketch to
512 dimensions and then projected to 64 dimensions by principal-component
analysis fitted on the real-training gradients. G-Vendi and BAIT use these
projected gradient features. LESS uses the 512-dimensional sketched gradients
before PCA.

\paragraph{Quality score.}
Candidate negative log-likelihood is the mean token loss of its response under
the checkpoint trained on real data, excluding prompt tokens. Lower values
indicate responses that the checkpoint considers more likely.

\subsection{Selection rules}

\paragraph{Random.}
Random samples candidate indices uniformly without replacement.

\paragraph{NLL.}
NLL sorts candidates by increasing response negative log-likelihood and
returns the first \(B\).

\paragraph{Embedding Vendi.}
Let \(\bm z_j\) be the unit-normalized prompt feature. For a nonempty set \(S\),
the Vendi score is
\begin{equation*}
 \operatorname{VS}(S)
 =\exp\!\left[-\sum_k\lambda_k\log\lambda_k\right],
 \qquad
 \{\lambda_k\}_k
 =\operatorname{eig}\!\left(
 \frac{1}{|S|}\sum_{j\in S}\bm z_j\bm z_j^\top
 \right),
\end{equation*}
where zero eigenvalues make zero contribution. A farthest-first traversal
initializes a set of size \(B\). We then run 32 one-swap rounds; each round
tests 256 proposed replacements and accepts the proposal with the largest
positive Vendi-score improvement.

\paragraph{G-Vendi.}
G-Vendi applies the same objective and swap procedure to unit-normalized
projected LoRA-gradient features. It replaces prompt representations with
candidate-induced training gradients while retaining the same set optimization.

\paragraph{LESS.}
Let \(\bm g_j\) be the candidate gradient feature and let
\(\overline{\bm g}_q\) be the mean gradient on the labeled real probe set. The score
is
\begin{equation*}
 s_{\mathrm{LESS}}(j)
 =\frac{\bm g_j^\top\overline{\bm g}_q}
 {\|\bm g_j\|_2\|\overline{\bm g}_q\|_2}.
\end{equation*}
Candidates are ranked by this target-gradient alignment score at the shared
checkpoint.

\paragraph{OptimSyn.}
For the fixed candidate pool, we adapt OptimSyn's optimizer-aware influence
score. For each candidate, we form the AdamW update difference
\(\Delta\bm\theta_j\) described in Appendix~\ref{app:optimizer} and score its
first-order effect on the real probe:
\begin{equation*}
 s_{\mathrm{OptimSyn}}(j)
 =-\frac{\bm g_q^\top\Delta\bm\theta_j}
 {\|\Delta\bm\theta_j\|_2}.
\end{equation*}

\paragraph{Covariance Matching.}
Normalize every prompt feature to unit Euclidean norm and define the
real-training second moment
\begin{equation*}
 \bm C_r=\frac{1}{|\mathcal D_r|}
 \sum_{x\in\mathcal D_r}\bm z_x\bm z_x^\top.
\end{equation*}
Starting from an empty set, the method greedily adds the candidate that minimizes
\begin{equation*}
 \left\|
 \frac{1}{|S|+1}
 \left(\sum_{i\in S}\bm z_i\bm z_i^\top
       +\bm z_j\bm z_j^\top\right)
 -\bm C_r
 \right\|_F^2.
\end{equation*}
This baseline matches the second moment of the real-training inputs and does
not use labeled probe responses.

\paragraph{BAIT.}
BAIT operates on the 64-dimensional projected gradient features. It initializes
its information matrix with the real-training gradients and uses the candidate
gradient distribution as the A-optimal target. A diagonal stabilizer equal to
\(10^{-3}\) times the average information-matrix diagonal is added before
inversion. Following the forward--backward construction, BAIT greedily grows
to \(2B\) candidates and then removes examples until \(B\) remain, choosing each
addition or removal by its A-optimal trace change.

\paragraph{TATC (Coverage).}
This ablation initializes \(\widehat{\bm A}_0\) and \(\widehat{\bm Q}\) exactly
as TATC but removes the training-effect filter. At every step it searches
the full remaining pool and adds
\begin{equation*}
 \arg\max_{j\notin S}
 \frac{\bm\phi_j^\top\widehat{\bm A}_S^{-1}
 \widehat{\bm Q}\widehat{\bm A}_S^{-1}\bm\phi_j}
 {1+\bm\phi_j^\top\widehat{\bm A}_S^{-1}\bm\phi_j}.
\end{equation*}
It isolates the contribution of the response-aware training-effect component
beyond target coverage.

\section{Implementation of TATC in Supervised Fine-tuning}
\label{app:math-details}

\subsection{From the Linear Analysis to TATC}
\label{app:nonlinear}

The linear results motivate two components of selection: the target coverage added by a candidate input and the training score of its complete prompt and response. TATC computes both at a checkpoint trained on real data. It ranks candidates using a quadratic score of their update differences and updates coverage as examples are selected.

For each prompt, we average the hidden states from the model's final transformer
layer over prompt tokens to obtain a 1,536-dimensional representation.  PCA is
fitted only to the 2,048 real-training representations and then applied,
without refitting, to the probe, holdout, and candidate prompts.  Let
\(\bm\phi_j\in\mathbb R^{64}\) be the resulting candidate feature and define
\begin{align*}
 \widehat{\bm A}_0
 &=\rho\bm I+
   \sum_{x\in\mathcal D_r}\bm\phi(x)\bm\phi(x)^\top,\\
 \widehat{\bm Q}
 &=\frac{1}{|\mathcal D_q|}
   \sum_{x\in\mathcal D_q}\bm\phi(x)\bm\phi(x)^\top,\\
 \widehat{\bm A}_S
 &=\widehat{\bm A}_0+
   \sum_{i\in S}\bm\phi_i\bm\phi_i^\top.
\end{align*}
The numerical stabilizer is
\begin{equation*}
 \rho=10^{-3}\frac{\operatorname{tr}
 (\sum_{x\in\mathcal D_r}\bm\phi(x)\bm\phi(x)^\top)}{64}.
\end{equation*}
The target coverage gain
\begin{equation*}
 a_j(S)=
 \frac{\bm\phi_j^\top\widehat{\bm A}_S^{-1}
 \widehat{\bm Q}\widehat{\bm A}_S^{-1}\bm\phi_j}
 {1+\bm\phi_j^\top\widehat{\bm A}_S^{-1}\bm\phi_j}
\end{equation*}
is exactly the Sherman--Morrison decrease in
\(\operatorname{tr}(\widehat{\bm Q}\widehat{\bm A}_S^{-1})\) within the fixed
prompt representation.  This algebraic identity is exact; using the criterion
as a proxy for downstream performance after LLM fine-tuning is an
experimental-design approximation.

The quadratic training score is
\begin{equation*}
 h_j=-\bm g_q^\top\Delta\bm\theta_j
 -\tfrac12\Delta\bm\theta_j^\top\bm H_q\Delta\bm\theta_j,
\end{equation*}
In Equation~\eqref{eq:local-value}, \(\Delta\bm\theta_j=\bm U_j-\bm U_{\varnothing}\) is the difference between the candidate and zero gradient AdamW displacements from the same checkpoint and saved optimizer state. The probe gradient \(\bm g_q\) measures alignment with decreasing loss, while the diagonal matrix \(\bm H_q\) supplies a quadratic penalty on displacement. The score provides a fixed ranking for the expanding shortlist. Appendix~\ref{app:optimizer} defines the quadratic model and its computation.

Table~\ref{tab:linear-tatc-map} summarizes the correspondence and its
limitations.

\begin{table}[H]
\centering
\small
\setlength{\tabcolsep}{3pt}
\caption{Correspondence between the linear analysis and TATC.}
\label{tab:linear-tatc-map}
\begin{tabular}{@{}
>{\raggedright\arraybackslash}p{.18\textwidth}
>{\raggedright\arraybackslash}p{.25\textwidth}
>{\raggedright\arraybackslash}p{.18\textwidth}
>{\raggedright\arraybackslash}p{.30\textwidth}@{}}
\toprule
Linear quantity or role & TATC quantity & Status & Correct interpretation \\
\midrule
Real information \(\bm A_r\) & \(\widehat{\bm A}_0\) & Approximation & Summarizes prompt directions already represented by real data; \(\rho\bm I\) is numerical stabilization. \\
Target moment \(\bm Q\) & \(\widehat{\bm Q}\) & Sample estimate & Estimates target relevance only in the checkpoint prompt representation. \\
Selected information \(\bm A_S\) & \(\sum_{i\in S}\bm\phi_i\bm\phi_i^\top\) & Exact in representation & Each selected input adds a rank-one prompt-feature contribution. \\
Clean coverage marginal \(q_j/c_j\) & \(a_j(S)\) & Exact in representation & Exact decrease in the downstream-weighted inverse-trace criterion. \\
Systematic label error \((\bm d_S,\bm u)\) & Response-dependent gradient and AdamW update & Approximation & For squared loss, response error changes the gradient by \(- (\bm x^\top\bm\delta)\bm x\); in an LLM the full response changes a nonlinear gradient. \\
Exact marginal value \(\Delta_j(S)\) & Quadratic training score \(h_j\) & Surrogate & Ranks update differences using a quadratic model of probe loss at the shared checkpoint. \\
Residual interaction \(e_j\) with current bias & No explicit scalar counterpart & Omitted interaction & The response score is computed once, while set dependence enters through \(a_j(S)\). \\
\bottomrule
\end{tabular}
\end{table}

\subsection{Quadratic training score}
\label{app:optimizer}

Let \(\bm g_j\) be the per-example gradient of the response-token loss and let
\(b_{\mathrm{eff}}=16\) be the effective batch size used to scale the candidate
gradient.  We insert \(\bm g_j/b_{\mathrm{eff}}\) into the saved AdamW state
after fine-tuning on real data.  For coordinate \(k\), with saved moments
\((m_{0,k},v_{0,k})\), the candidate and zero-gradient counterfactual moments
are
\begin{align*}
 m_{j,k}&=\beta_1m_{0,k}+(1-\beta_1)g_{j,k}/b_{\mathrm{eff}},
 &v_{j,k}&=\beta_2v_{0,k}+(1-\beta_2)(g_{j,k}/b_{\mathrm{eff}})^2,\\
 m_{\varnothing,k}&=\beta_1m_{0,k},
 &v_{\varnothing,k}&=\beta_2v_{0,k}.
\end{align*}
After the usual AdamW bias corrections, let
\begin{equation*}
 U_k(m,v)=-\alpha\frac{\widehat m_k}{\sqrt{\widehat v_k}+\epsilon}-\alpha\lambda\theta_{0,k}
\end{equation*}
denote the full parameter displacement at \(\bm\theta_0\), where \(\alpha\) is the learning rate and \(\lambda\) is the weight decay coefficient.  Let \(\bm U_j\) and \(\bm U_{\varnothing}\)
collect these coordinatewise updates under the candidate and zero-gradient
moments.  Their difference is
\begin{equation*}
 \Delta\theta_{j,k}=U_k(m_j,v_j)-U_k(m_{\varnothing},v_{\varnothing}).
\end{equation*}
The common weight decay displacement \(-\alpha\lambda\bm\theta_0\) cancels in \(\Delta\bm\theta_j\). The saved optimizer moments still affect this difference through the adaptive updates.

We compute the probe statistics using batches of four examples. Let \(\{\mathcal B_b\}_{b=1}^{M}\) partition the probe examples across all scoring workers, let \(n_b=|\mathcal B_b|\), and let \(\ell_b(\bm\theta)\) be the model's response token loss for batch \(\mathcal B_b\). Set \(w_b=n_b/|\mathcal D_q|\) and define
\begin{align*}
 L_q(\bm\theta)&=\sum_{b=1}^{M}w_b\ell_b(\bm\theta),
 &\bm g_{q,b}&=\nabla_{\bm\theta}\ell_b(\bm\theta_0),\\
 \bm g_q&=\sum_{b=1}^{M}w_b\bm g_{q,b},
 &[\bm H_q]_{kk}&=\sum_{b=1}^{M}w_b[\bm g_{q,b}]_k^2.
\end{align*}
The implementation squares each batch gradient before averaging across batches. We use the resulting positive semidefinite diagonal matrix as a curvature surrogate in the scoring model~\citep{kunstner2019empiricalfisher}.

For a displacement \(\bm d\) from \(\bm\theta_0\), define the quadratic model
\begin{equation*}
 \widetilde L_q(\bm d)
 =L_q(\bm\theta_0)+\bm g_q^\top\bm d
 +\tfrac12\bm d^\top\bm H_q\bm d.
\end{equation*}
Evaluating this model along the update difference gives
\begin{align*}
 h_j
 &=\widetilde L_q(\bm 0)-\widetilde L_q(\Delta\bm\theta_j)\\
 &=-\bm g_q^\top\Delta\bm\theta_j
 -\tfrac12\Delta\bm\theta_j^\top\bm H_q\Delta\bm\theta_j.
\end{align*}
The first term measures alignment of the update difference with decreasing probe loss. The second is a nonnegative penalty on displacement, since \(\bm H_q\succeq0\). The quadratic model scores \(\bm\theta_0+\Delta\bm\theta_j\); the full AdamW step reaches \(\bm\theta_0+\bm U_j\). The proposition below bounds the loss change along \(\Delta\bm\theta_j\).

Both probe statistics and candidate update differences are computed once at the shared checkpoint. Each candidate therefore has one fixed score during selection. The implementation accumulates both terms directly from the update difference and the probe statistics. CountSketch is used only for gradient features stored for the baselines.

\paragraph{Local justification of the score.}
The displacement \(\Delta\bm\theta_j\) used for scoring is the update difference defined above. The following result relates the score in Equation~\eqref{eq:local-value} to the change in probe loss along this displacement and separates the error in the curvature matrix from the Taylor remainder.

\begin{proposition}[Local loss approximation]
\label{prop:local-score-accuracy}
Suppose \(L_q\) is twice continuously differentiable in a neighborhood of the segment from \(\bm\theta_0\) to \(\bm\theta_0+\Delta\bm\theta_j\), and its Hessian is \(M\)-Lipschitz there in spectral norm. Let
\[
 \bm H_*=\nabla^2 L_q(\bm\theta_0),\qquad
 \epsilon_H=\|\bm H_*-\bm H_q\|_2,
\]
and use \(\bm g_q=\nabla L_q(\bm\theta_0)\) in the score \(h_j\). Then
\begin{equation}
 \left|L_q(\bm\theta_0)-L_q(\bm\theta_0+\Delta\bm\theta_j)-h_j\right|
 \leq
 \frac{\epsilon_H}{2}\|\Delta\bm\theta_j\|_2^2
 +\frac{M}{6}\|\Delta\bm\theta_j\|_2^3.
\label{eq:local-score-error}
\end{equation}
In particular, the displacement reduces probe loss whenever \(h_j\) exceeds the bound in Equation~\eqref{eq:local-score-error}.
\end{proposition}

\begin{proof}
Write \(\bm d=\Delta\bm\theta_j\). Taylor's theorem with a Lipschitz Hessian gives
\[
 L_q(\bm\theta_0+\bm d)
 =L_q(\bm\theta_0)+\bm g_q^\top\bm d
 +\tfrac12\bm d^\top\bm H_*\bm d+R(\bm d),
 \qquad |R(\bm d)|\leq\tfrac{M}{6}\|\bm d\|_2^3.
\]
Subtracting the definition of \(h_j\) from the loss reduction yields
\[
 L_q(\bm\theta_0)-L_q(\bm\theta_0+\bm d)-h_j
 =-\tfrac12\bm d^\top(\bm H_*-\bm H_q)\bm d-R(\bm d).
\]
The bound follows from
\( |\bm d^\top(\bm H_*-\bm H_q)\bm d|\leq\epsilon_H\|\bm d\|_2^2 \).
\end{proof}

The quadratic model therefore gives a local scoring rule whose accuracy depends on both the displacement size and the curvature approximation. The first term in the bound measures the error from the curvature surrogate supplied by the batch gradient construction; the cubic term controls the Taylor remainder.

\subsection{Sequential TATC selection}
\label{app:algorithm}

We implement Algorithm~\ref{alg:tatc} with \(\kappa=2\) and update
\(\widehat{\bm A}_S^{-1}\) after each selection using Sherman--Morrison:
\begin{equation*}
 \widehat{\bm A}_{S\cup\{j\}}^{-1}
 =\widehat{\bm A}_S^{-1}
 -\frac{\widehat{\bm A}_S^{-1}\bm\phi_j\bm\phi_j^\top
        \widehat{\bm A}_S^{-1}}
 {1+\bm\phi_j^\top\widehat{\bm A}_S^{-1}\bm\phi_j}.
\end{equation*}
One run produces an ordering whose prefixes are the selected sets for every
budget.  Each prefix is then trained independently from the same checkpoint
trained on real data and the same
saved optimizer state; training does not continue from the preceding budget.

\subsection{GSM8K protocol}

\paragraph{Data.}

The task is GSM8K mathematical reasoning with
Qwen2.5-Math-1.5B-Instruct.  The disjoint real partitions contain 2,048
examples used to train the shared checkpoint, 512 labeled examples for the
target probe, and 1,024 examples for the holdout set.  The synthetic pool contains 4,096
OpenR1-Math-220k trajectories after deduplication to at most one trajectory per
source problem.  Of these trajectories, 2,662 pass the provided verifier and
1,434 do not; TATC never uses verifier status.

\paragraph{Model and training.}
LoRA adapters are attached to the attention \(q,k,v,o\) projections and the
MLP gate, up, and down projections.  Rank is 16, scaling is 32, and adapter
dropout is zero.  Fine-tuning on real data uses two epochs, learning rate
\(2\times10^{-4}\), AdamW weight decay \(.01\), sequence length 2,304, and
effective batch size 16.  Every selected subset is trained for one epoch with
the same learning rate, weight decay, effective batch size, and initial model
and optimizer state. We evaluate budgets \(B\in\{64,128,256,512\}\) using
nested prefixes of the selected ordering. Results are the mean and standard
deviation over three independent training repetitions.

\paragraph{Evaluation.}
Exact match is computed after extracting the final numeric answer. Outputs
that do not terminate within the generation allowance count as incorrect.
The \(B=128\) comparison allows 1,536 generated tokens, whereas the budget
sweep allows 3,072; absolute exact-match values are therefore not comparable
across panels.

\end{document}